\documentclass[10pt]{article}

\usepackage[top=0.9in,bottom=1.1in,left=1.05in,right=1.05in]{geometry}

\usepackage{amsmath}
\usepackage{graphicx}
\usepackage{xr-hyper}
\usepackage[colorlinks=true, allcolors=blue, urlcolor=black]{hyperref}
\usepackage{upgreek}
\usepackage{amssymb}
\usepackage{amsfonts}
\usepackage{amsthm}
\usepackage{mathrsfs}
\usepackage{fontenc}
\usepackage{empheq}
\usepackage{enumerate}
\usepackage{authblk}
\usepackage{comment}
\usepackage{appendix}
\usepackage{mathtools}
\usepackage[shortlabels]{enumitem}
\usepackage{algorithm}
\usepackage[ruled,algo2e]{algorithm2e}
\usepackage{bm}

\mathtoolsset{showonlyrefs=true}
\DeclareMathOperator*{\argmin}{arg\,min} 

\let\inf\relax \DeclareMathOperator*\inf{\vphantom{p}inf}

\theoremstyle{plain}
\newtheorem{thm}{Theorem}
\newtheorem{lem}{Lemma}
\newtheorem{prop}{Proposition}

\newtheorem*{prop:repeat1}{Proposition \ref{thm: concentration_coefficient}}
\newtheorem*{prop:repeat2}{Proposition \ref{lem_2}}

\theoremstyle{definition}
\newtheorem{defn}{Definition}

\theoremstyle{remark}
\newtheorem{rem}{Remark}

\newcommand{\email}[1]{\protect\href{mailto:#1}{#1}}
\newcommand{\bR}{\mathbb {R}}
\newcommand{\bH}{H}
\newcommand{\bP}{\mathbb{P}}
\newcommand{\bS}{\mathcal{S}}
\newcommand{\bA}{\mathcal{A}}
\newcommand{\EE}{\mathbb{E}}
\newcommand{\cB}{\mathcal{B}}
\newcommand{\cH}{\mathcal{H}}
\newcommand{\cond}{\,|\,}
\newcommand{\transpose}{^{\operatorname{T}}}

\newcommand{\rmd}{\,\mathrm{d}}

\title{Perturbational Complexity by Distribution Mismatch: A Systematic Analysis of Reinforcement Learning in Reproducing Kernel Hilbert Space}

\author[1]{Jihao Long\thanks{\email{jihaol@princeton.edu}}}
\author[2]{Jiequn Han\thanks{\email{jiequnhan@gmail.com}}}
\affil[1]{Program of Applied and Computational Mathematics, Princeton University}
\affil[2]{Center for Computational Mathematics, Flatiron Institute}

\begin{document}

\maketitle
\begin{abstract}
 Most existing theoretical analysis of reinforcement learning (RL) is limited to the tabular setting or linear models due to the difficulty in dealing with function approximation in high dimensional space with an uncertain environment. This work offers a fresh perspective into this challenge by analyzing RL in a general reproducing kernel Hilbert space (RKHS). We consider a family of Markov decision processes $\mathcal{M}$ of which the reward functions lie in the unit ball of an RKHS and transition probabilities lie in a given arbitrary set. We define a quantity called \textit{perturbational complexity by distribution mismatch} $\Delta_{\mathcal{M}}(\epsilon)$ to characterize the complexity of the admissible state-action distribution space in response to a perturbation in the RKHS with scale $\epsilon$. We show that $\Delta_{\mathcal{M}}(\epsilon)$ gives both the lower bound of the error of all possible algorithms and the upper bound of two specific algorithms (fitted reward and fitted Q-iteration) for the RL problem. Hence, the decay of $\Delta_\mathcal{M}(\epsilon)$ with respect to $\epsilon$ measures the difficulty of the RL problem on $\mathcal{M}$. We further provide some concrete examples and discuss whether $\Delta_{\mathcal{M}}(\epsilon)$ decays fast or not in these examples. As a byproduct, we show that when the reward functions lie in a high dimensional RKHS, even if the transition probability is known and the action space is finite, it is still possible for RL problems to suffer from the curse of dimensionality.
    
\end{abstract}

\section{Introduction}
\label{sec_intro}
Modern reinforcement learning (RL) algorithms in practice often utilize function approximation tools to deal with problems involving an enormous amount of states in high dimensions. However, the majority of existing theoretical analysis of RL is only applicable to the tabular setting~(see, e.g.,~\cite{jaksch2010near,azar2012sample,osband2016generalization,azar2017minimax,dann2017unifying,jin2018q}), in which both the state and action spaces are discrete and finite, and no function approximation is involved. Relatively simple function approximation methods, such as 
the linear model in \cite{yang2019sample,jin2020provably} or generalized linear model in \cite{wang2019optimism,li2021sample}, have been recently studied in the context of RL with various statistical estimates. 
Yet, these results are not sufficient to explain the practical success of RL algorithms in high dimensions.
In the tabular setting, the number of samples required by an RL algorithm is proportional to the size of state-action pairs, which is enormous in practice. For linear model or generalized linear model, the assumption therein is pretty restrictive for practice.
The kernel function is a class of models that can approximate more general functions than the tabular setting or linear model, and it is widely used in practice.
Moreover, the kernel function approximation is closely related to neural network approximation, as established in the theory of neural tangent kernel in \cite{jacot2018neural} and Barron space in \cite{ma2019barron}.
RL with kernel function approximation has been recently studied in~\cite{farahmand2016regularized,domingues2020regret,yang2020provably,yang2020function,long20212}. Still, results therein either suffer from the curse of dimensionality or require stringent assumptions on the kernel or the dynamics. 
As pointed out in \cite{kuo2008multivariate} and  \cite[Section~5]{long20212}, $L^\infty$-estimation, as a widely used technique in classical RL analysis of the tabular setting, may give rise to a curse of dimensionality in kernel method, which signifies new difficulties in RL algorithms with kernel function approximation. 
In this paper, we aim at a systematic study of RL with kernel function approximation and consider the following question:\\

\emph{When can a reinforcement learning problem be solved efficiently using kernel function approximation?}
\\

Note that in this paper ``efficiency" is considered in terms of sample complexity, i.e., the number of data an algorithm needs to collect to achieve a specified performance criterion. Our analysis will focus on the sample complexity of the RL problem, while other complexities such as computational complexity will not be covered.
We are particularly interested in high dimensional state-action spaces and want to identify RL problems that can be solved efficiently with kernel function approximation even in high dimensions.
Similar questions have been studied in the realm of supervised learning. There the answer is quite clear: once the target function lies in a reproducing kernel Hilbert space (RKHS), no matter how large the dimension is, the corresponding supervised learning problem can be solved efficiently (see, e.g.,~\cite{scholkopf2002learning}). In RL, the reward function plays a similar role with the target function in supervised learning. This analogy motivates us to study a more concrete question:\\

\emph{If the unknown reward function lies in an RKHS, what is the condition of the RKHS and transition dynamics to ensure that the reinforcement learning problem can be solved efficiently?}
\\

Below we give some intuition of the main challenge to answer this question and the key concept we introduce in this paper.
Given an RKHS $\mathcal{H}$ and a probability distribution $\nu$, existing results in supervised learning have shown that for any target distribution $g$ lying in the unit ball of $\mathcal{H}$, one can efficiently obtain an estimation $\hat{g}$ such that
\begin{equation*}
    \|g - \hat{g}\|_\mathcal{H} \le 2, \|g - \hat{g}\|_{L^2(\nu)} \le \epsilon
\end{equation*}
for any $\epsilon > 0$ (see, e.g.,~\cite{scholkopf2002learning} or Lemma \ref{lem_1}). We can then view $\hat{g}$ as a $\nu$-perturbation of $g$ and define the $\nu$-perturbation space with scale $\epsilon$ as
\begin{equation*}
   \mathcal{H}_{\epsilon,\nu} = \{ g \in \mathcal{H}\colon \|g\|_\mathcal{H} \le 1, \|g\|_{L^2(\nu)}\le \epsilon\}.
\end{equation*}
While the distribution $\nu$ is given in supervised learning, in the theoretical analysis of RL,
one needs to control the difference of the expectation between the target function and estimation with a probability distribution unknown a priori. 
That probability distribution is unknown because that is the state distribution or state-action distribution induced by a particular policy, which is unknown a priori.
Take the estimation of the optimal Q-value function $Q_h^*$ (see \eqref{def_optimal_q} below for the detailed definition) for example. 
The optimal policy can be derived from the optimal Q-value function through the greedy policy. In practice, given any probability distribution $\nu$, under certain conditions, one can estimate the optimal Q-value function $Q_h^*$ in the sense of $L^2(\nu)$ using Q-learning algorithm in \cite{cai2019neural} or fitted Q-iteration algorithm (see Algorithm \ref{alg:FittedQ} or  \cite{fan2020theoretical} and \cite{long20212}). In other words, one can obtain $\hat{Q}^*_h$, a $\nu$-perturbation of $Q^*_h$. However, when evaluating the performance of $\hat{\pi}$, the greedy policy derived from $\hat{Q}_h^*$, one needs to control the error between $Q_h^*$ and $\hat{Q}_h^{*}$ under the state-action distribution induced by the policy $\hat{\pi}$ (see the performance difference lemma in~\cite{kakade2002approximately} or \eqref{performance_dif}), which is unknown before one obtains $\hat{Q}_h^*$. We call this phenomenon \textit{distribution mismatch}: mismatch between the distribution $\nu$ for estimation and the distribution for evaluation that is unknown a priori. This phenomenon is ubiquitous in the analysis of RL (see, e.g.,~\cite[Section~6]{kakade2002approximately}).
Although not detailed above, when estimating the optimal Q-value function in the sense of $L^2(\nu)$, one needs to deal with the error propagation between steps, and distribution mismatch also brings difficulty.

To quantify the error brought by distribution mismatch, we define a semi-norm
\begin{equation*}
    \|g\|_\Pi = \sup_{\rho \in \Pi} |\int g \rmd \rho|
\end{equation*}
where $\Pi$ is a set of probability distributions and introduce the \textit{perturbation response by distribution mismatch}: 
\begin{equation*}
    \mathcal{R}(\Pi,\mathcal{H},\epsilon,\nu) = \sup_{g \in \mathcal{H}_{\epsilon,\nu}}\|g\|_\Pi.
\end{equation*}
One shall notice that if $\Pi=\{\nu\}$, then $\mathcal{R}(\Pi,\mathcal{H},\epsilon,\nu)$ cannot be greater than $\epsilon$. However, in analysis, we usually can only choose $\Pi$ as the possible state-action distributions under a class of policies.
The scale of perturbation response by distribution mismatch measures the discrepancy between $\nu$ and $\Pi$ and reflects the error brought by the fact that we do not know the state-action distribution under the policy of interest. If $\Pi$ consists of all probability distributions, then the above semi-norm is just the $L^\infty$-norm, which is used to handle the distribution mismatch in the tabular and linear RL problems.
However, for many common RKHSs, the $L^\infty$-estimation may suffer from the curse of dimensionality; see \cite{kuo2008multivariate} and \cite[Section~5]{long20212} for a detailed discussion. The challenge of $L^\infty$-estimation in high dimensional space reveals the difficulty of RL problems in the RKHS compared to the tabular setting or linear function approximation.
In this sense, the introduced $\Pi$-norm can be understood as a generalization of the $L^\infty$-norm to overcome this difficulty. This concept takes into account the distribution structure of the RL problem and allows us to do a more delicate error analysis.
Following this idea, we introduce the \textit{perturbational complexity by distribution mismatch}  $\Delta_\mathcal{M}(\epsilon)$ for a large class of families of Markov decision processes (MDPs) and prove that once the perturbational complexity decreases fast with respect to $\epsilon$, the RL problem can be solved efficiently.
On the other hand, by considering the RL problem in which one only knows the reward function lies in the unit ball of a general RKHS and transition probability lies in a given arbitrary set, we show that the perturbational complexity $\Delta_\mathcal{M}(\epsilon)$ must decay fast with respect to $\epsilon$ if this RL problem can be solved efficiently.

Combining the above two types of results together, we show that the perturbational complexity $\Delta_\mathcal{M}(\epsilon)$ measures the intrinsic difficulty of an RL problem.
Note that most of our results still hold if we replace RKHS with a Banach space in which we can efficiently obtain an $L^2$-estimation, e.g., linear space or Barron space in \cite{ma2019barron}.
Furthermore, our results shed some light on studying practical RL algorithms.
First,  the structure of $\Pi$ has been used in the previous analysis of RL in various settings; see e.g.,~\cite{farahmand2010error,farahmand2016regularized,chen2019information,agarwal2021theory}. While previous works mainly focus on the so-called concentration coefficients of $\Pi$ and use related assumptions to prove upper bounds for RL problems, our work shows the necessity of additional assumptions on $\Pi$ in order to ensure that the RL problem in the RKHS can be efficiently solved.
As indicated by Proposition \ref{lem_2}, if the eigenvalue decay of the kernel is slow and $\Pi$ consists of all probability distributions, then $\Delta_\mathcal{M}(\epsilon)$ also decays slowly. Therefore, to design efficient RL algorithms, one needs to better understand the set $\Pi$, particularly when the eigenvalue decay of the kernel is slow.
Second, when the unknown reward function lies in the unit ball of an RKHS and the transition probability is known, Theorems \ref{thm: Known Transition} and \ref{upper_bound_known} show that solving the RL problem is equivalent to using finite values of a target function $g$ to obtain a function estimate $\hat{g}$ that is accurate with respect to the $\Pi$-norm; see Remark \ref{equivalence_2} for detailed discussions. 
Theorems \ref{thm: unknown_transition} and \ref{upper_bound_unknown} also establish a partial connection between these two problems in the case of unknown transition probability. Therefore, it is helpful to study this supervised learning problem as a prototype of the RL problem. %
\newline

\noindent\textbf{Our Contribution}
\begin{enumerate}
    \item We define the perturbational complexity by distribution mismatch $\Delta_\mathcal{M}(\epsilon)$ for the families of MDPs $\mathcal{M}$ of which the reward functions lie in the unit ball of an RKHS and transition probabilities lie in a given arbitrary set.  We then show that $\Delta_\mathcal{M}(\epsilon)$ gives a lower bound for the error of every algorithm on the corresponding RL problem (Theorems \ref{thm: Known Transition} and \ref{thm: unknown_transition}).
    \item  In the case of known transition (all transition probabilities in the families of MDPs are the same), we show that $\Delta_\mathcal{M}(\epsilon)$ also gives an upper bound of the error of the fitted reward algorithm (Algorithm \ref{alg:FittedReward}) without any further assumption (Theorem \ref{upper_bound_known}).
    \item In the case of unknown transition (general case), with an additional assumption on Bellman operators \eqref{Bellman_assumption}, we show that $\Delta_\mathcal{M}(\epsilon)$ gives an upper bound for the error of the fitted Q-iteration algorithm (Algorithm \ref{alg:FittedQ} and Theorem \ref{upper_bound_unknown}).
    \item We give a concrete form of the perturbation response by distribution mismatch (Lemma \ref{thm: concentration_coefficient}) and show that when the assumptions on concentration coefficients in the existing literature (see e.g., \cite{farahmand2016regularized,chen2019information,fan2020theoretical,long20212}) are satisfied or the eigenvalue decay of the kernel is fast, $\Delta_\mathcal{M}(\epsilon)$ decays fast with respect to $\epsilon$ (Proposition \ref{concentrated_case} and Proposition \ref{lem_2}).
    \item We give a concrete example in which the reward functions lie in a high dimensional RKHS, the transition probability is known, and the action space is finite, but the corresponding RL problem can not be solved without the curse of dimensionality (Proposition \ref{cod_case}).
\end{enumerate}

\noindent\textbf{Related Literature}
While the optimal lower bound of the error of RL algorithms in the tabular setting has been established in \cite{azar2012sample,azar2017minimax}, there are much fewer results about lower bounds of RL with function approximation. \cite{ni2019learning} proves an optimal lower bound for Lipschitz function approximation. \cite{du2019good} shows that even when the value function, policy function, reward function, and transition probability can be approximated by a linear function, it is still possible that solving the RL problem requires samples exponentially depending on the horizon. \cite{chen2019information} shows that even when the set of candidate approximating functions is finite and includes the optimal Q-value function, there does not exist an algorithm whose sample size is a polynomial function of the logarithm of the size of the candidate function set, the size of action space, horizon, and the reciprocal of accuracy. In other words, the previous works either consider function spaces (Lipschitz function space) that are too large to derive meaningful upper bound or only give lower bounds on special cases. Instead, we consider a fairly general class of RL problems associated with the RKHS and give both lower bound and upper bound through the perturbational complexity by distribution mismatch. 

Previous works establish several upper bounds for RL algorithms with kernel function approximation. Based on the type of used assumptions, these works can be divided into two categories. The first category of upper bounds in \cite{domingues2020regret,yang2020provably,yang2020function} depends on the eigenvalue decay of kernel while the second category in \cite{farahmand2010error,long20212} requires accessibility to reference distributions that can uniformly bound all possible state-action distributions under admissible policies (called assumption on concentration coefficients). In this work, we show that the perturbational complexity $\Delta_\mathcal{M}(\epsilon)$ decays fast in both situations and establish an upper bound for the fitted reward algorithm (see Algorithm \ref{alg:FittedReward} in Section~\ref{sec:upper_known}) and the fitted Q-iteration algorithm (see Algorithm \ref{alg:FittedQ} in Section~\ref{sec:upper_unknown}) under the assumption that $\Delta_\mathcal{M}(\epsilon)$ decays fast. In this sense, our work generalizes both categories of the previous work. %

Besides the error bounds of the RL algorithms, there is recent work in \cite{duan2021optimal} studying  policy evaluation in RKHS as a component of the RL algorithm and analyzing its optimal convergence rate.
\vspace{2.5ex}\newline
\noindent\textbf{Notation}: Let $\mathcal{X}$ be an arbitrary subset of a Euclidean space, we use $C(\mathcal{X})$ and $\mathcal{P}(\mathcal{X})$ to denote the bounded continuous function space and probability distribution space on $\mathcal{X}$, respectively. We use $\|\cdot\|_{C(\mathcal{X})}$ to denote the uniform norm on $C(\mathcal{X})$:
$$
    \|g\|_{C(\mathcal{X})} = \sup_{x \in \mathcal{X}}|g(x)|.
$$
Given a probability distribution $\nu$ on $\mathcal{X}$, we use $\|\cdot\|_{L^2(\nu)}$ and $\|\cdot\|_{\infty}$ to denote $L^2$-norm and $L^\infty$-norm, respectively.
Given two probability distributions $\mu$ and $\nu$ in $\mathcal{P}(\mathcal{X})$, define the total variation distance:
\begin{equation*}
    \|\mu-\nu\|_{TV} = \sup\{ |\mu(A) - \nu(A)|\colon A \text{ is a measurable subset of }\mathcal{X}\}.
\end{equation*}
When $\mu$ is absolute continuous with respect to $\nu$, define
 the Radon-Nikodym derivative $\frac{\rmd \mu}{\rmd \nu}$ 
and the Kullback-Leibler divergence:
\begin{equation*}
    \mathrm{KL}(\mu\,||\,\nu) = \int_{\mathcal{X}}\log(\frac{\rmd \mu}{\rmd \nu})\rmd \mu.
\end{equation*}
For any random variable, we use $\mathcal{L}(\cdot)$ to denote the law of the random variable.
Given a positive integer $H$, $[H]$ denotes the set $\{1,\dots,H\}$. $\mathbb{N}^+$ denotes the set of all positive integers. $\mathbb{S}^{d-1}$ denotes the unit sphere of $\bR^d$: $\{x \in \bR^d, \|x\|_2 = 1\}$.
Given a Banach space $\mathcal{B}$, we use  $\|\cdot\|_\mathcal{B}$ to denote  the norm of $\mathcal{B}$.
We say $f(n) = \Theta(g(n))$, if there exist two constants $c, C > 0$ independent of $n$ such that $cg(n) \le f(n) \le Cg(n)$.\footnote{Later we also use $\Theta$ to denote an index set associated with a family of MDPs. The specific meaning should be always clear from context.}

\section{Preliminary}
\subsection{Markov Decision Process}
We consider an episodic MDP $(\bS,\bA,H,P,r,\mu)$ as the mathematical model for the RL problem. Here $H$ is a positive constant integer indicating the length of each episode. $\bS$ and $\bA$ denote the set of all the states and actions, respectively. 
We assume $\bS$ is a subset of a Euclidean space and $\bA$ is a compact subset of a Euclidean space.
$P\colon [\bH]\times\bS\times\bA \mapsto \mathcal{P}(\bS)$ is the state transition probability. For each $(h,s,a) \in [\bH]\times\bS\times\bA$, $P(\,\cdot\cond h,s,a)$ denotes the transition probability for the next state at step $h$ if the current state is $s$ and action $a$ is taken. $r\colon [\bH] \times \bS \times\bA \mapsto \mathbb{R}$ is the reward function, denoting the expected reward at step $h$ if we choose action $a$ at the state $s$. We assume each observed reward is the sum of the expected reward and an independent standard Gaussian noise. $\mu \in \mathcal{P}(\bS)$ is the initial distribution.

We denote a policy by $\pi = \{\pi_h\}_{h=1}^H \in \mathcal{P}(\bA \cond \bS, H)$, where
\begin{align}\label{def_col_distribution}
    \mathcal{P}(\bA\,|\,\bS,H) = \Big\{\{\pi_h(\,\cdot \cond \cdot\,)\}_{h=1}^H\colon \pi_h(\,\cdot \cond s) \in \mathcal{P}(\bA)
    \text{ for any }s \in \bS \text{ and } h \in [H]\Big\}.
\end{align}
Given a time step $h$, a transition probability $P$, a policy $\pi$ and an initial distribution $\mu$, we use $\rho_{h,P,\pi,\mu}$ to denote the distribution of $(S_h,A_h)$ where $S_1 \sim \mu$, $A_h$ follows the policy $\pi_h(\,\cdot\cond S_h)$ and $S_{h+1}$ is distributed according to the transition probability $P(\,\cdot\cond h,S_h,A_h)$. Moreover, we use $\Pi(h,P,\mu)$ to denote the set of all the possible distributions of $\rho_{h,P,\pi,\mu}$ as follows
\begin{equation*}
    \Pi(h,P,\mu) = \{\rho_{h,P,\pi,\mu}\colon \pi \in \mathcal{P}(\bA \cond \bS,H)\}.
\end{equation*}
and let
\begin{equation*}
    \Pi(P,\mu) = \bigcup_{h \in [H]}\Pi(h,P,\mu).
\end{equation*}
Given an MDP $M$ and a policy $\pi$, we define the total reward as follows:
\begin{equation*}
    J(M,\pi) = \sum_{h=1}^H \int_{\bS\times\bA} r(h,s,a)\rmd \rho_{h,P,\pi,\mu}(s,a).
\end{equation*}
The optimal total reward is defined as $J^*(M) = \sup_{\pi \in \mathcal{P}(\bA \cond \bS,H)} J(M,\pi)$.
We assume there exists at least one optimal policy $\pi^*$ such that $J(M,\pi^*)=J^*(M)$.

\subsection{Reproducing Kernel Hilbert Space (RKHS)}
Suppose $k:(\bS\times\bA)\times(\bS\times\bA)\mapsto \bR$ is a continuous positive definite kernel that satisfies 
\begin{enumerate}
    \item $k(z,z') = k(z',z)$, $\forall z, z' \in \bS\times\bA$;
    \item $\forall\, m \ge 1$, $z_1,\dots,z_m \in \bS\times\bA$ and $c_1,\dots,c_m \in \bR$, we have:
    \begin{equation*}
        \sum_{i=1}^m\sum_{j=1}^mc_ic_jk(z_i,z_j) \ge 0.
    \end{equation*}
\end{enumerate}
Then, there exists a Hilbert space  $\mathcal{H}_k \subset C(\bS\times\bA)$ such that
\begin{enumerate}
    \item $\forall\, z \in \bS\times\bA$, $k(z,\,\cdot\,) \in \mathcal{H}_k$;
    \item $\forall\, z \in \bS\times\bA$ and $g \in \mathcal{H}_k$, $g(z) = \langle g, k(z,\,\cdot\,)\rangle_{k}$,
\end{enumerate}
and $k$ is called the reproducing kernel of $\mathcal{H}_{k}$ in \cite{aronszajn1950theory} and we use $\|\cdot\|_{k}$ and $\langle\,\cdot, \cdot\, \rangle_k$ to denote the norm and inner product in the Hilbert space $\mathcal{H}_k$, repsectively.

Given a probability distribution $\nu$ on $\bS \times \bA$, we will use $\{\Lambda_i^{\nu}\}_{i \in \mathbb{N}^{+}}$ and $\{\psi_i^{\nu}\}_{i \in \mathbb{N}^{+}}$ to denote the eigenvalues and eigenfunctions of the operator
\begin{equation*}
    (\mathcal{K}_{\nu}g) (z) \coloneqq \int_{\bS \times \bA}k(z,z')g(z')\rmd \nu(z') 
\end{equation*}
from $L^2(\nu)$ to $L^2(\nu)$. We futher require that $\{\Lambda_i^\nu\}_{i \in \mathbb{N}^{+}}$ is nonincreasing and $\{\psi_i^\nu\}_{i \in \mathbb{N}^{+}}$ is orthonormal in $L^2(\nu)$. The famous Mercer decomposition states that
\begin{equation}\label{mercer_decompo}
    k(z,z') = \sum_{i=1}^{+\infty}\Lambda_i^\nu \psi_i^\nu(z)\psi_i^{\nu}(z').
\end{equation}
Moreover, for any $g \in \mathcal{H}_k$
\begin{equation}\label{mercer_norm}
    \|g\|_k^2 = \sum_{i=1}^{+\infty}\frac{1}{\Lambda_i^\nu}|\langle g, \psi_i^\nu\rangle_{L^2(\nu)}|^2.
\end{equation}
See, e.g., \cite[Section~2.1]{bach2017equivalence}.

Given any two probability distributions $\rho$ and $\rho'$ on $\bS\times\bA$,  the maximum mean discrepancy (MMD) is defined as follows (see e.g. \cite{borgwardt2006integrating}):
\begin{equation*}
    \mathrm{MMD}_k(\rho,\rho') = \sup_{\|g\|_k \le 1}|\int_{\bS\times\bA}g(z)\rmd \rho(z) - \int_{\bS\times\bA}g(z)\rmd \rho'(z)|.
\end{equation*}
An equivalent but more concrete expression of MMD is
\begin{equation*}
    \mathrm{MMD}_k(\rho,\rho') = \sqrt{\int_{\bS\times\bA}\int_{\bS\times\bA}k(z,z')\rmd(\rho-\rho')(z)\rmd(\rho-\rho')(z')}.
\end{equation*}

\section{Problem Setup}
We first specify our prior knowledge of the RL problem. We want to solve an RL problem whose underlying MDP belongs to a family of MDPs 
\begin{equation*}
    \mathcal{M}= \{M_\theta = (\bS, \bA, P_\theta, r_\theta, H, \mu)\colon \theta \in \Theta\}
\end{equation*}
where $\bS$,  $\bA$, $H$ and $\mu$ are common state space, action space, length of each episode and initial distribution. The possible transition probability $P_\theta$ and reward function $r_\theta$ is indexed by $\theta$, and $\Theta$ is an index set. 
We do not know the exact value of $\theta$ but can access a generative simulator. In other words, for any step $h \in [H]$ and state-action pair $(s,a)$, we can observe a state $x \sim P_\theta(\,\cdot\,|\,h,s,a)$ and a noisy reward $y \sim \mathcal{N}(r_\theta(h,s,a),1)$, which is called one sample or one access to the generative simulator. So far we need to assume the noise of the reward is Gaussian to prove the lower bounds, but the noise can be relaxed to be sub-Gaussian in the upper bounds. Another popular form of the simulator is the so-called episodic simulator, through which one can only choose the initial state and a policy to observe the whole path and corresponding rewards. Our lower bound is still valid if we only have an episodic simulator but might be loose. How to obtain a tight lower bound in those settings is left to future work.

We assume $\theta=(\theta_P, \theta_r)$ and the index set is a Cartesian product
\begin{equation*}
    \Theta = \{(\theta_P, \theta_r) \colon \theta_P\in \Theta_P, \theta_r\in \Theta_r\}
\end{equation*}
where $\theta_P$ and $\theta_r$ are the actual indexes of the transition probability and reward function, \text{i.e.}, $P_\theta=P_{\theta_P}$, $r_\theta=r_{\theta_r}$.
We also assume
\begin{align*}
    \{r_{\theta_r}\colon \theta_r\in\Theta_r\} = \{r\colon  \|r(h,\cdot,\cdot)\|_\mathcal{B} \le 1, \forall h \in [H]\}
\end{align*}
where $\mathcal{B}$ is a Banach space such that $\mathcal{B}$ is a subset of $C(\bS \times \bA)$ and $\|\cdot\|_{C(\bS\times\bA)} \le B\|\cdot\|_{\mathcal{B}}$ with a positive constant $B$. 

Following the intuition introduced in Section~\ref{sec_intro}, we give the following definitions in preparation for the analysis.
\begin{defn}~
\begin{enumerate}[label={\upshape(\roman*)}, widest=iii]
\item 
For any set $\Pi$ consisting of probability distributions on $\bS\times\bA$, we define a semi-norm $\|\cdot\|_\Pi$ on $C(\bS\times\bA)$:
\begin{equation*}
    \|g\|_\Pi \coloneqq \sup_{\rho \in \Pi}|\int_{\bS\times\bA} g(s,a)\rmd\rho(s,a)|.
\end{equation*}
We call this semi-norm {\emph{$\Pi$-norm}}.

\item 
Given a Banach space, a positive constant $\epsilon > 0$ and a probability distribution $\nu \in \mathcal{P}(\bS\times\bA)$, we define $\mathcal{B}_{\epsilon,\nu}$, a {\emph{$\nu$-perturbation space with scale $\epsilon$}}, as follows:
\begin{equation*}
    \mathcal{B}_{\epsilon,\nu} \coloneqq \{g \in \mathcal{B}\colon \|g\|_\mathcal{B} \le 1, \|g\|_{L^2(\nu)} \le \epsilon\}.
\end{equation*}

\item
The \emph{perturbation response by distribution mismatch} is defined as the radius of $\mathcal{B}_{\epsilon,\nu}$ under $\Pi$-norm,
\begin{equation*}
    \mathcal{R}(\Pi,\mathcal{B},\epsilon,\nu) \coloneqq \sup_{g \in \mathcal{B}_{\epsilon,\nu}} \|g\|_\Pi.
\end{equation*}
\end{enumerate}
\end{defn}

{
\subsection{Properties of Perturbation Response by Distribution Mismatch}
We first state two propositions later used to give readers more understanding of the properties of perturbation response by distribution mismatch $\mathcal{R}(\Pi,\mathcal{B},\epsilon,\nu)$. The proofs of these two propositions are postponed to Section \ref{Sec_concentra}. 
The first proposition gives a more concrete formula of $\mathcal{R}(\Pi,\mathcal{B},\epsilon,\nu)$. Specifically, when $\mathcal{B}$ is an RKHS, $\mathcal{R}(\Pi,\mathcal{B},\epsilon,\nu)$ can be determined by a maximin problem related to MMD; see~\eqref{eq:kernel_dual_repeat}.
\begin{prop:repeat1}
We have
\begin{equation}
   \mathcal{R}(\Pi,\cB,\epsilon,\nu) = \sup_{\rho \in \Pi}\inf_{g \in L^2(\nu)}[\|\rho - g\circ \nu\|_{\mathcal{B}^{*}} + \epsilon\|g\|_{L^2(\nu)}],
\end{equation}
where $g\circ \nu$ is a signed measure such that 
\begin{equation*}
    \frac{\rmd g \circ \nu}{\rmd \nu} = g,
\end{equation*}
$\mathcal{B}^*$ is the dual space of $\mathcal{B}$
and $\|\rho\|_{\mathcal{B}^{*}}$ is the dual norm of linear functional 
\begin{equation*}
    \rho(g) \coloneqq \int_{\bS\times\bA}g(z)\rmd \rho(z),\,\forall g \in \mathcal{B},
\end{equation*}
for any signed measure $\rho$ on $\bS\times\bA$ (we slightly abuse the notation that $\rho$ are both the signed measure and linear functional in $\mathcal{B})$.
If $\mathcal{B}$ is an RKHS with kernel k, then
\begin{align}
    \mathcal{R}(\Pi,\cH_k,\epsilon,\nu) = \sup_{\rho \in \Pi}\inf_{g \in L^2(\nu)}
   [\mathrm{MMD}_k(\rho, g\circ \nu) + \epsilon\|g\|_{L^2(\nu)}].
\label{eq:kernel_dual_repeat}
\end{align} 

\end{prop:repeat1}

When $\mathcal{B}$ is an RKHS, the kernel's eigenvalues encode much information. The following proposition shows that the perturbation response of $\mathcal{P}(\bS\times\bA)$, the set of all probability distributions on $\bS\times\bA$, is closely related to the kernel's eigenvalues. Later in Section~\ref{Sec_concentra} we will discuss how this proposition gives us the implication in the efficiency of RL algorithms.
\begin{prop:repeat2}
Assume that 
\begin{equation*}
    \sup_{z \in \bS \times \bA} k(z,z) \le 1.
\end{equation*}
For any $\rho \in \mathcal{P}(\bS\times\bA)$, define
\begin{equation}
    n(\rho) =  \max\{ i \in \mathbb{N}^+: n \Lambda_i^{\rho} \ge 1\}.
\end{equation}
We have
\begin{equation}
   \mathcal{R}(\mathcal{P}(\bS\times\bA),\cH_k,n^{-\frac{1}{2}},\nu) \ge \frac{1}{2}\sqrt{\sup_{\rho \in \mathcal{P}(\bS\times \bA)}\sum_{i = n(\nu)+1}^{+\infty}\Lambda_i^{\rho}},
\end{equation}
and, by $n(\nu) \le n$,
\begin{equation}
    \inf_{\nu \in \mathcal{P}(\bS\times\bA)}\mathcal{R}(\mathcal{P}(\bS\times\bA),\cH_k,n^{-\frac{1}{2}},\nu) \ge \frac{1}{2}\sqrt{\sup_{\rho \in \mathcal{P}(\bS\times \bA)}\sum_{i = n+1}^{+\infty}\Lambda_i^{\rho}}.
\end{equation}
Moreover, if there exists a distribution $\hat{\nu} \in \mathcal{P}(\bS\times\bA)$ such that 
\begin{equation}
    \sup_{ i \in \mathbb{N}^+}\|\psi_i^{\hat{\nu}}\|_\infty < +\infty,
\end{equation}
then
\begin{equation}
    \mathcal{R}(\mathcal{P}(\bS\times\bA),\cH_k,n^{-\frac{1}{2}},\hat{\nu})\le2 \sqrt{\frac{n(\hat{\nu})}{n}+ \sum_{i=n(\hat{\nu})+1}^{_\infty}\Lambda_i^{\hat{\nu}}}\sup_{ i \in \mathbb{N}^+}\|\psi_i^{\hat{\nu}}\|_\infty .
\end{equation}
\end{prop:repeat2}
}

\subsection{General Algorithm}

Now we state in Algorithm~\ref{alg:general} the general RL algorithm for estimating the optimal value $J^*(M_{\theta})$ with $n$ samples. 
In Algorithm~\ref{alg:general}, the superscript $\theta$ indicates that the collected data depends on the underlying MDP $M_\theta$.
The superscript $\xi$ denotes the collection $\{f_1,\dots, f_n,F\}$, where $f_i$ are measurable mappings: $([H]\times\bS \times \bA \times \bS \times \bR)^{\otimes i-1}\times \bR \mapsto [H]\times\bS \times \bA $, $F$ is a measurable mapping: $([H]\times\bS \times \bA \times  \bS \times \bR)^{\otimes n}\times \bR \mapsto \bR$.
$\xi$ can be viewed as an RL algorithm, which adaptively chooses the step-state-action tuple $(h,s,a)$ at the $i$-th step based on all received data $\mathcal{D}_{i-1}^{\theta,\xi}$ according to function $f_i$ and receives a subsequent state and reward through the generative simulator. After collecting $n$ samples, the algorithm outputs an estimate of the optimal value based on all data according to function $F$. The randomness of the whole process in Algorithm~\ref{alg:general} is related to i.i.d. standard normal random variables $\{\epsilon_i\}_{1\le i \le n}$, $\{u_i\}_{1\le i\le n}$, and $\bar{u}$, which all live in a common probability space $(\Omega, \bP)$. $\epsilon_i$ denotes the noise in the observed reward. $u_i$ denotes the randomness of the transition, for which we assume that, by the isomorhism theorem \cite[Section~41]{halmos2013measure}, $p_\theta : [H]\times\bS \times \bA \times \mathbb{R} \mapsto \bS$ is a measurable function satisfying
$p_\theta(h,s,a,u_i) \sim P_\theta(\,\cdot\,|\,h,s,a)$ for any $\theta \in \Theta$ and $h \in [H]$.
Again by the isomorphism theorem, we use $\bar{u}$ to denote all the randomness of the algorithm $\xi$ itself besides the randomness within the simulator.
Mathematically, Algorithm~\ref{alg:general} can also be summarized as follows:
\begin{align}\label{Sampling_path}
\begin{cases}
    \mathcal{D}_0^{\theta,\xi} = \emptyset,\;
    \mathcal{D}_i^{\theta,\xi} = \mathcal{D}_{i-1}^{\theta,\xi} \cup \{(h_i^{\theta,\xi},s_i^{\theta,\xi},a_i^{\theta,\xi},x_i^{\theta,\xi},y_i^{\theta,\xi})\}, \, 1\le i \le n,\;
    J_n^{\theta,\xi} = F(\mathcal{D}_n^{\theta,\xi}, \bar{u}),\\
    (h_i^{\theta,\xi},s_i^{\theta,\xi},a_i^{\theta,\xi}) = f_i(\mathcal{D}_{i-1}^{\theta,\xi},\bar{u}),\;
    x_i^{\theta,\xi} = p_\theta(h_i^{\theta,\xi},s_i^{\theta,\xi},a_i^{\theta,\xi},u_i),\;
    y_i^{\theta,\xi} = r_\theta(h_i^{\theta,\xi},s_i^{\theta,\xi},a_i^{\theta,\xi})+\epsilon_i.
\end{cases}
\end{align}

\begin{algorithm}[ht]
\caption{General Reinforcement Learning Algorithm for Estimating the Optimal Value}
{
\KwIn{Number of samples $n$}
}
{
\textbf{Initialize:} $\mathcal{D}_0^{\theta,\xi} = \emptyset$. \\ 
}
\For{$i = 1,\dots,n$}{
Obtain $i$-th step-state-action tuple through $(h_i^{\theta,\xi},s_i^{\theta,\xi},a_i^{\theta,\xi}) = f_i(\mathcal{D}_{i-1}^{\theta,\xi},\bar{u})$\\
Collect the subsequent state $x_i^{\theta,\xi} = p_\theta(h_i^{\theta,\xi},s_i^{\theta,\xi},a_i^{\theta,\xi},u_i)$ and the
 noisy reward $y_i^{\theta,\xi} = r_\theta(h_i^{\theta,\xi},s_i^{\theta,\xi},a_i^{\theta,\xi})+\epsilon_i$ from the simulator\\
Set $\mathcal{D}_i^{\theta,\xi} = \mathcal{D}_{i-1}^{\theta,\xi} \cup \{(h_i^{\theta,\xi},s_i^{\theta,\xi},a_i^{\theta,\xi},x_i^{\theta,\xi},y_i^{\theta,\xi})\}$
 }

\KwOut{$J_n^{\theta,\xi} = F(\mathcal{D}_n^{\theta,\xi}, \bar{u})$ as an estimate of the optimal value $J^*(M_{\theta})$
}
\label{alg:general}
\end{algorithm}

We use $\Xi_n$ to denote the set of all possible choices of $\xi$. So $\Xi_n$ is the set of all possible RL algorithms which only access the generative simulator $n$ times.
Our goal is to find the best $\xi$, or the best RL algorithm, to minimize the worst-case error of the optimal total reward given $n$ opportunities to access the simulator:
\begin{equation*}
\label{eq:worst-case-error}
    \inf_{\xi \in \Xi_n}\sup_{\theta \in \Theta}\EE|J_n^{\theta,\xi} - J^*(M_\theta)|.
\end{equation*}
In Sections~\ref{sec:lb} and~\ref{sec:ub} below, we give lower and upper bounds for the worst-case error, respectively. In both sections, we first consider the special case where the transition probability is known and then generalize our results to the case where the transition probability is unknown.
In practice, it is often of interest to estimate the optimal policy as well.
In the upper bound part, we also provide algorithms to obtain the optimal policy that gives the estimated optimal total reward. Nevertheless, in the lower bound part, we abstractly estimate the optimal total reward without estimating the optimal policy.
Note that we can always use the Monte-Carlo method to estimate the optimal total reward accurately given an optimal or near-optimal policy. So our lower bound result still serves as a valid difficulty measure of the RL problem aiming at finding the optimal policy.

\section{Lower Bound}
\label{sec:lb}
\subsection{The Case of Known Transition}
We first consider the case that the transition probability $P_\theta$ is known, assuming $\Theta_P = \{0\}$ is a single-point set. In this case, we have the following definition of the perturbational complexity by distribution mismatch.

\begin{defn}
\label{def:delta1}
The \emph{perturbational complexity by distribution mismatch} in the case of known transition is
\begin{equation}
\label{eq:definition_delta1}
  \Delta_\mathcal{M}(\epsilon) \coloneqq \inf_{\nu \in \mathcal{P}(\bS\times\bA)} \mathcal{R}(\Pi(P_0,\mu),\mathcal{B},\epsilon,\nu).
\end{equation}
\end{defn}

The following theorem shows that this quantity can give a lower bound of the worst-case error. 

\begin{thm}\label{thm: Known Transition}
Assume $\Theta_P = \{0\}$, i.e., there is only one possible transition probability,
then
\begin{equation*}
   \inf_{\xi \in \Xi_n}\sup_{\theta \in \Theta} \bP(|J_n^{\theta,\xi} - J^*(M_\theta)| \ge \frac{1}{3}\Delta_\mathcal{M}(n^{-\frac{1}{2}}))\ge \frac{1}{4}.
\end{equation*}
Therefore,
\begin{equation*}
    \inf_{\xi \in \Xi_n}\sup_{\theta \in \Theta} \EE|J_n^{\theta,\xi} - J^*(M_\theta)| \ge \frac{1}{12}\Delta_\mathcal{M}(n^{-\frac{1}{2}}).
\end{equation*}
\end{thm}
\begin{proof}
With fixed $g \in \mathcal{B}$ such that $\|g\|_\mathcal{B} \le 1$, $\xi \in \Xi_n$ and $h^* \in [H]$, we first estimate
\begin{equation*}
    \|\mathcal{L}(J_n^{\theta_1,\xi}) - \mathcal{L}(J_n^{\theta_2,\xi})\|_{TV},
\end{equation*}
where $\theta_1, \theta_2 \in \Theta$ satisfying
\begin{equation*}
    P_{\theta_1} = P_{\theta_2} = P_0,\; r_{\theta_1} = 0, \; r_{\theta_2}(h,s,a) = \begin{cases}0, \text{ when }h \neq h^*\\
    g(s,a), \text{ when } h = h^*
    \end{cases}
\end{equation*}

By the definition \eqref{Sampling_path} and Pinsker's inequality (see e.g. \cite{pinsker1964information,csiszar1967information}), we have
\begin{equation*}%
     \|\mathcal{L}(J_n^{\theta_1,\xi}) - \mathcal{L}(J_n^{\theta_2,\xi})\|_{TV}^2 \le \|\mathcal{L}(\mathcal{D}_n^{\theta_1,\xi},\bar{u})-\mathcal{L}(\mathcal{D}_n^{\theta_2,\xi},\bar{u})\|_{TV}^2 \le \frac{1}{2} \mathrm{KL}(\mathcal{L}(\mathcal{D}_n^{\theta_1,\xi},\bar{u})||\mathcal{L}(\mathcal{D}_n^{\theta_2,\xi},\bar{u})).
\end{equation*}

Calculation gives that
\begin{align*}
    &\mathrm{KL}(\mathcal{L}(\mathcal{D}_n^{\theta_1,\xi},\bar{u})||\mathcal{L}(\mathcal{D}_n^{\theta_2,\xi},\bar{u})) \notag \\
    =~& \EE\log(\prod_{i=1}^n\exp{(-\frac{(y_i^{\theta_1,\xi})^2}{2}+\frac{(y_i^{\theta_1,\xi}-r_{\theta_2}(h_i^{\theta_1,\xi},s_i^{\theta_1,\xi},a_i^{\theta_1,\xi})^2}{2})})\notag\\ 
    =~& \frac{1}{2}\sum_{i=1}^n\EE[|g(s_i^{\theta_1,\xi},a_i^{\theta_1,\xi})|^2 - 2g(s_i^{\theta_1,\xi},a_i^{\theta_1,\xi})\epsilon_i]\mathrm{1}_{h_i^{\theta_1,\xi} = h^*} \notag\\
    \le~& \frac{1}{2}\sum_{i=1}^n\EE[|g(s_i^{\theta_1,\xi},a_i^{\theta_1,\xi})|^2].
\end{align*}
Combining the last two inequalities, we know that
\begin{equation*}
     \|\mathcal{L}(J_n^{\theta_1,\xi}) - \mathcal{L}(J_n^{\theta_2,\xi})\|_{TV}^2 \le \frac{1}{4}\sum_{i=1}^n\EE[ g(s_i^{\theta_1,\xi},a_i^{\theta_1,\xi})]^2 = \frac{n}{4}\int_{\bS \times \bA} |g(s,a)|^2\rmd \nu^{\xi}(s,a),
\end{equation*}
where
\begin{equation*}
    \nu^{\xi} = \frac{1}{n}\sum_{i=1}^n\mathcal{L}(s_i^{\theta_1,\xi},a_i^{\theta_1,\xi}).
\end{equation*}
By definition, we have
\begin{equation*}
    \Delta_\mathcal{M}(n^{-\frac{1}{2}}) \le\sup_{h \in [H]} \sup_{\|g\|_\mathcal{B}\le 1,\sqrt{n}\|g\|_{L^2(\nu^\xi)} \le 1}\sup_{\rho \in \Pi(h,P_0,\mu)}\int_{\bS\times\bA}g(s,a)\rmd\rho(s,a).
\end{equation*}
So we can choose $g \in \mathcal{B}$ with $\|g\|_\mathcal{B} \le 1$ and $h^* \in [H]$ to define a reward function $r_{\theta_2}$ such that
\begin{align*}
    &\int_{\bS \times \bA} |g(s,a)|^2\rmd \nu^{\xi}(s,a) \le \frac{1}{n}, \\
    &J^*(M_{\theta_2}) \ge \frac{2}{3}\Delta_\mathcal{M}(n^{-\frac{1}{2}}),
\end{align*}
which means that
\begin{align*}
     &\|\mathcal{L}(J_n^{\theta_1,\xi}) - \mathcal{L}(J_n^{\theta_2,\xi})\|_{TV}\le \frac{1}{2}, \\
     &\{x \in \bR \colon |x- J^*(M_{\theta_2})|< \frac{1}{3}\Delta_\mathcal{M}(n^{-\frac{1}{2}})\}\cap \{x \in \bR \colon |x| < \frac{1}{3}\Delta_\mathcal{M}(n^{-\frac{1}{2}})\} = \emptyset.
\end{align*}
Noticing that $J^*(M_{\theta_1})=0$, we can use the second condition above to have
\begin{align*}
    &\bP(|J_n^{\theta_1,\xi} - J^*(M_{\theta_1})|\ge\frac{1}{3}\Delta_\mathcal{M}(n^{-\frac{1}{2}}) \\
    =~&\bP(|J_n^{\theta_1,\xi}|\ge \frac{1}{3}\Delta_\mathcal{M}(n^{-\frac{1}{2}}) \\
    \ge~&\bP(|J_n^{\theta_1,\xi} - J^*(M_{\theta_2})| < \frac{1}{3}\Delta_\mathcal{M}(n^{-\frac{1}{2}})).
\end{align*}
Therefore, 
\begin{align*}
    \frac{1}{2} &\ge \bP(|J_n^{\theta_2,\xi} - J^*(M_{\theta_2})| < \frac{1}{3}\Delta_\mathcal{M}(n^{-\frac{1}{2}})) -\bP(|J_n^{\theta_1,\xi} - J^*(M_{\theta_2})| < \frac{1}{3}\Delta_\mathcal{M}(n^{-\frac{1}{2}})) \notag\\
    &\ge 1-  \bP(|J_n^{\theta_2,\xi} - J^*(M_{\theta_2})| \ge \frac{1}{3}\Delta_\mathcal{M}(n^{-\frac{1}{2}})) - \bP(|J_n^{\theta_1,\xi} - J^*(M_{\theta_1})|\ge \frac{1}{3}\Delta_\mathcal{M}(n^{-\frac{1}{2}})).
\end{align*}
Rearranging the above inequality, we can then conclude that for any $\xi \in \Xi_n$,
\begin{equation*}
    \sup_{\theta \in \Theta}\bP(|J_n^{\theta,\xi} - J^*(M_\theta)| \ge \frac{1}{3}\Delta_\mathcal{M}(n^{-\frac{1}{2}})) \ge \frac{1}{4}.
\end{equation*}
\end{proof}

\subsection{The Case of Unknown Transition}
In this section, we deal with the general case when there are multiple possible transition probabilities. 
Following the idea of Theorem \ref{thm: Known Transition}, 
\begin{equation*}
    \sup_{\theta \in \Theta}\Delta_{\mathcal{M}_\theta}(n^{-\frac{1}{2}})
\end{equation*}
can provide a lower bound of the worst-case error, where $\mathcal{M}_\theta$ is the subset of $\mathcal{M}$ whose transition probability is $P_\theta$.
However, we can have an even better lower bound. The critical observation is that we do not know the exact value of $P_\theta$, but can only sample from $P_\theta$ with finite observed data. 
So the optimal distribution $\nu$ for estimation corresponding to $\Delta_{\mathcal{M}_\theta}(n^{-\frac{1}{2}})$ is generally inaccessible, and we can leverage this fact to better characterize the perturbational complexity to improve the lower bound.
To this end, we assume a general sampling algorithm as follows to characterize the set of distribution $\nu$:
\begin{align}
\begin{cases}\label{def_sample_path_2}
    \mathcal{D}_0^{\theta,\bar{\xi}} = \emptyset,\; \mathcal{D}_i^{\theta,\bar{\xi}} = \mathcal{D}_{i-1}^{\theta,\bar{\xi}} \cup \{(h_i^{\theta,\bar{\xi}},s_i^{\theta,\bar{\xi}},a_i^{\theta,\bar{\xi}},x_i^{\theta,\bar{\xi}})\}, \, 1\le i \le n,   \\
    (h_i^{\theta,\bar{\xi}},s_i^{\theta,\bar{\xi}},a_i^{\theta,\bar{\xi}}) = \bar{f}_i(\mathcal{D}_{i-1}^{\theta,\bar{\xi}},\bar{u}),\;
    x_i^{\theta,\bar{\xi}}= p_\theta(h_i^{\theta,\bar{\xi}},s_i^{\theta,\bar{\xi}},a_i^{\theta,\bar{\xi}},u_i), \\
    \displaystyle{\nu^{\theta,\bar{\xi}} = \frac{1}{n}\sum_{i=1}^n \mathcal{L}(s_i^{\theta,\bar{\xi}}, a_i^{\theta,\bar{\xi}})}.
\end{cases}
\end{align}
Here $\bar{\xi} = (\bar{f}_1,\dots,\bar{f}_n)$ and $\bar{f}_i$ are measurable mappings $([H]\times\bS\times\bA \times \bS)^{\otimes (i-1)} \times\bR \mapsto [H]\times \bS \times \bA$.
Equation \eqref{def_sample_path_2} shares a similar spirit with \eqref{Sampling_path}, but only focuses on sampling the next-step states from the transition probability. Similar to $\xi$ in \eqref{Sampling_path}, $\bar{\xi}$ can be viewed as a sampling algorithm which adaptively chooses $n$ step-state-action tuples $(h,s,a)$ and obtains samples from $P_\theta(\,\cdot\,|\,h,s,a)$. 
We let $\bar{\Xi}_n$ denote the set of all possible choices of $\bar{\xi}$.
Now we define the perturbational complexity by distribution mismatch in the case of unknown transition.

\begin{defn}
\label{def:delta2}
The \emph{perturbational complexity by distribution mismatch} in the case of unknown transition is
\begin{equation}
\label{eq:definition_delta2}
 \Delta_{\mathcal{M}}(\epsilon) \coloneqq\inf_{\bar{\xi}\in \bar{\Xi}_{[1/\epsilon^2]}}\sup_{\theta \in \Theta}\mathcal{R}(\Pi(P_\theta,\mu),\mathcal{B},\epsilon,\nu^{\theta,\bar{\xi}}).
\end{equation}
\end{defn}
Here we choose the number of samples to be $[1/\epsilon^2]$ so that the result is consistent with the case of known transition. The following theorem shows that perturbational complexity gives a lower bound of the worst-case error in the case of unknown transition.

\begin{thm}\label{thm: unknown_transition}
We have
\begin{equation*}
    \inf_{\xi \in \Xi_n}\sup_{\theta \in \Theta} \bP(|J_n^{\theta,\xi} - J^*(M_\theta)| \ge \frac{1}{3}\Delta_\mathcal{M}(n^{-\frac{1}{2}}))\ge \frac{1}{4}.
\end{equation*}
Therefore,
\begin{equation*}
    \inf_{\xi \in \Xi_n}\sup_{\theta \in \Theta} \EE|J_n^{\theta,\xi} - J^*(M_\theta)| \ge \frac{1}{12}\Delta_\mathcal{M}(n^{-\frac{1}{2}}).
\end{equation*}
\end{thm}

\begin{proof}
Following the proof of Theorem \ref{thm: Known Transition}, we know that for any $\xi \in \Xi_n$,
\begin{equation*}
    \sup_{\theta \in \Theta} \bP\left(|J_n^{\theta,\xi} - J^*(M_\theta)| \ge \frac{1}{3}\sup_{\theta\in \Theta}\mathcal{R}(\Pi(P_\theta,\mu),\mathcal{B},n^{-\frac{1}{2}},\nu^{\theta,\xi})\right)\ge \frac{1}{4}.
\end{equation*}
Here
\begin{equation*}
    \nu^{\theta,\xi} = \frac{1}{n}\sum_{i=1}^n\mathcal{L}(s_i^{\theta_0(\theta),\xi},a_i^{\theta_0(\theta),\xi}),
\end{equation*}
and $\{s_i^{\theta_0(\theta),\xi},a_i^{\theta_0(\theta),\xi}\}_{1\le i\le n}$ is generated by the sampling path \eqref{Sampling_path} with $P_{\theta_0(\theta)} = P_\theta$ and $r_{\theta_0(\theta)} = 0$. Hence, $y_i^{\theta_0(\theta),\xi} = \epsilon_i$ for any $1\le i \le n$. Using the isomorphism theorem \cite[Section~41]{halmos2013measure}, we can find the measurable mappings $T_{i}: \bR \mapsto \bR$ for $1\le  i \le n +1$ such that
\begin{equation*}
   (T_1(\bar{u}),\dots,T_{n}(\bar{u}),t_{n+1}(\bar{u})) \text{ has the same distribution with }(\epsilon_1,\dots,\epsilon_n,\bar{u}).
\end{equation*}
Therefore, for any $\xi \in \Xi_n$, there exists $\bar{\xi} \in \bar{\Xi}_n$ such that for all $\theta \in \Theta$,
\begin{equation*}
    \nu^{\theta,\bar{\xi}} = \nu^{\theta,\xi}.
\end{equation*}
Therefore,
\begin{equation*}
    \sup_{\theta\in \Theta}\mathcal{R}(\Pi(P_\theta,\mu),\mathcal{B},n^{-\frac{1}{2}},\nu^{\theta,\xi}) \ge \Delta_\mathcal{M}(n^{-\frac{1}{2}}),
\end{equation*}
which concludes our proof.
\end{proof}

\section{Upper Bound}
\label{sec:ub}
In this section, we discuss how to use $\Delta_\mathcal{M}(\epsilon)$ in Definitions~\ref{def:delta1}, \ref{def:delta2} to design sample-efficient RL algorithms.
We will use $C$ to denote a universal positive constant, which may vary from line to line.
As motivated in the introduction, one important reason for us to consider the perturbation response is to study those high-dimensional spaces in which $L^\infty$ estimation can not be obtained efficiently through finite samples. Many common RKHSs are such examples. So in this section we focus on the case that $\mathcal{B}$ is an RKHS with kernel $k$. We remark that most of our results still hold for general Banach spaces whenever an $L^2$ estimation can be obtained efficiently through finite samples, such as linear space and Barron space \cite{ma2019barron}.
\subsection{The Case of Known Transition}\label{sec:upper_known}
Again we first consider the case where the transition probability is known, i.e., $\Theta_P = \{0\}$. 
We let
\begin{equation}
    \hat{\nu} = \argmin_{\nu \in \mathcal{P}(\bS\times\bA)}\mathcal{R}(\Pi(P_0,\mu),\mathcal{H}_k,n^{-\frac{1}{2}},\nu).
\end{equation}
Here $\hat{\nu}$ is an intrinsic property of the MDP family $\mathcal{M}$ without any dependence of sampling. In practice, the minimizer $\hat{\nu}$ is probably hard to obtain but similar argument holds if we can obtain a probability distribution $\hat{\nu}'$ such that
\begin{equation*}
   \mathcal{R}(\Pi(P_0,\mu),\mathcal{H}_k,n^{-\frac{1}{2}},\hat{\nu}')
\end{equation*}
is small. Once we know the distribution $\hat{\nu}$, we can sample $n^2$ i.i.d. samples $z_1,\dots,z_{n^2}$ from $\hat{\nu}$ and perform the following fitted reward algorithm to estimate the optimal policy. Notice that, in order to solve~\eqref{optimization_problem1} below, we only need to solve a convex optimization problem with the same objective function but in the finite-dimensional set
\begin{equation*}
    \{\|r\|_k \le 1 \colon  r = \sum_{i=1}^{n^2} c_ik(\cdot,z_i) \}.
\end{equation*}
The reason is that for any $r \in \mathcal{H}_k$, we can always find coefficients $\{c_{i}\}_{1\le i \le n^2}$ to construct a function $\sum_{i=1}^{n^2}c_{i}k(\cdot,z_i)$ such that
\begin{align*}
    r(z_j) = \sum_{i=1}^{n^2}c_{i}k(z_j,z_i),  {1\le j \le n^2},\, \text{ and  }
    \|r\|_k \ge \|\sum_{i=1}^{n^2}c_{i}k(\cdot,z_i)\|_k.
\end{align*}
See, e.g.,~\cite[Proposition~4.2]{paulsen2016introduction}.

\begin{algorithm}[ht]
\caption{Fitted Reward Algorithm}
{
\KwIn{$n^2$ i.i.d. samples $z_1,\dots,z_{n^2}$ from distribution $\hat{\nu}$}
}

\For{$h = 1,2,\dots,H$}{

Sample $y_1^{\theta,h},\dots,y_{n^2}^{\theta,h}$ from $\mathcal{N}(r_\theta(h,z_1),1),\dots,\mathcal{N}(r_\theta(h,z_{n^2}),1)$, respectively\\
Compute $\hat{r}_\theta(h,\cdot)$ as the minimizer of the optimization problem
\begin{equation}
\label{optimization_problem1}
    \min_{\|r\|_k \le 1} \sum_{i=1}^{n^2}[r(z_i) - y_i^{\theta,h}]^2
\end{equation}
}
Collect the fitted reward function to form the MDP $(\bS,\bA,H, P_0,\hat{r}_\theta,\mu)$, of which both reward function and transition are known. Denote it as $\hat{M}_{\theta}$.

\KwOut{
$\hat{\pi}_\theta$ as the optimal policy of $\hat{M}_{\theta}$.
}
\label{alg:FittedReward}
\end{algorithm}

The algorithm described above based on the fitted reward is summarized in Algorithm~\ref{alg:FittedReward}, and we have the following convergence result.
\begin{thm}\label{upper_bound_known}
Assume
$
   \Theta_P = \{0\},
$ i.e., there is only one possible transition probability, and 
\begin{equation*}
\sup_{z \in \bS\times \bA} k(z,z) \le 1.    
\end{equation*}
For any $\theta \in \Theta$ and $p \in (0,1)$, with probability at least $1-p$, we have
\begin{equation*}
    |J(M_\theta,\hat{\pi}_\theta) - J^*(M_\theta)|\le CH\Delta_\mathcal{M}(n^{-\frac{1}{2}})\sqrt{1+\log(\frac{nH}{p})}.
\end{equation*}
\end{thm}
\begin{rem}\label{rem_no_match}
As we need $Hn^2$ samples to achieve the upper bound $\Delta_\mathcal{M}(n^{-\frac{1}{2}})$, the convergence rate with respect to $n$ in this theorem does not match the lower bound offered in Theorem \ref{thm: Known Transition}. Still, $\Delta_\mathcal{M}(n^{-\frac{1}{2}})$ does give an indication whether an RL problem can be solved efficiently or not. For example, if $\Delta_\mathcal{M}(n^{-\frac{1}{2}}) = \Theta(n^{-\frac{1}{2}})$, then the convergence rate with respect to $n$ is between $n^{-\frac{1}{4}}$ and $n^{-\frac{1}{2}}$. We can hence know that the corresponding RL problem can be solved efficiently. On the other hand, if and only if $\Delta_{\mathcal{M}_d}(n^{-\frac{1}{2}})= \Theta(n^{-\frac{1}{d}})$, where $(\mathcal{M}_d)_{d \in \mathbb{N}^+}$ are families of MDPs and $d$ denotes the dimension of state-action space of $\mathcal{M}_d$, we know that the RL problems suffer from the curse of dimensionality. Overall, whether we can establish a dimension-free convergence rate for $\Delta_\mathcal{M}(n^{-\frac{1}{2}})$ determines whether we can construct a dimension-free RL algorithm. 
Similar arguments hold for the case of unknown transition; see Theorems \ref{thm: unknown_transition} and \ref{upper_bound_unknown}.
\end{rem}
Before proving Theorem~\ref{upper_bound_known}, we first prove the following lemma concerning the $L^2$-distance between $r_\theta$ and $\hat{r}$ solved in Algorithm~\ref{alg:FittedReward}.

\begin{lem}\label{lem_1}
Assume that 
\begin{equation*}
\sup_{z \in \bS\times \bA} k(z,z) \le 1.    
\end{equation*}
Let $(z_{1,1},\dots,z_{1,n}),\dots,(z_{n,1},\dots,z_{n,n})$ be i.i.d. drawn from a distribution $\nu \in (\bS\times\bA)^{\otimes n}$ (the n-ary Cartesian power of $\bS\times\bA$) and
\begin{equation*}
    \bar{\nu} = \frac{1}{n}\sum_{i=1}^n \nu_i,
\end{equation*}
where $\{\nu_i\}_{1\le i \le n}$ are marginal distributions of $\nu$. 
Random variables $\{\epsilon_{i,j}\}_{1\le i \le n,1 \le j \le n}$, conditional on $\{z_{i,j}\}_{1\le i \le n, 1\le j \le n}$, are independent, 1-subguassian and mean zero.
Given $r^* \in \mathcal{H}_k$ with $\|r^*\|_k \le 1$,  we let
\begin{equation}\label{def_hat_r}
    \hat{r} = \argmin_{\|r\|_k \le 1}\sum_{i=1}^n\sum_{j=1}^n[r(z_{i,j}) - r^*(z_{i,j}) - \epsilon_{i,j}]^2.
\end{equation}
Then, for any $p \in (0,1)$, with probability at least $1-p$, we have
\begin{equation*}
    \|\hat{r} - r^*\|_{L^2(\bar{\nu})} \le C\sqrt{\frac{1}{n}[1 + \log(\frac{n}{p})]}.
\end{equation*}
\end{lem}
\begin{rem}\label{equivalence_1}
Note that the setting of Lemma \ref{lem_1} is more general than Theorem \ref{upper_bound_known} since we do not require $z_{i,1},\dots,z_{i,n}$ to be i.i.d. random variables and we know the error is Gaussian instead of sub-Gaussian. We will use this general setting later to deal with the case of unknown transition; see the proof of Theorem~\ref{upper_bound_unknown}.
\end{rem}

\begin{proof}
Following~\eqref{def_hat_r}, we know that
\begin{equation*}
    \sum_{i=1}^n\sum_{j=1}^n[\hat{r}(z_{i,j}) - r^*(z_{i,j}) - \epsilon_{i,j}]^2 \le \sum_{i=1}^n\sum_{j=1}^n\epsilon_{i,j}^2.
\end{equation*}
Therefore,
\begin{align}
    \sum_{i=1}^n\sum_{j=1}^n[\hat{r}(z_{i,j}) - r^*(z_{i,j})]^2 \le 2\sum_{i=1}^n\sum_{j=1}^n\epsilon_{i,j}[\hat{r}(z_{i,j})-r^*(z_{i,j})]
    \le 4\sup_{\|r\|_k \le 1}\sum_{i=1}^n\sum_{j=1}^n \epsilon_{i,j} r(z_{i,j}).
\end{align}
Notice that
\begin{equation*}
    \sup_{\|r\|_k \le 1}\sum_{i=1}^n\sum_{j=1}^n \epsilon_{i,j} r(z_{i,j}) = \sup_{\|r\|_k \le 1}\langle r, \sum_{i=1}^n\sum_{j=1}^n \epsilon_{i,j} k(\cdot,z_{i,j})\rangle_k = \sqrt{\sum_{i=1}^n\sum_{j = 1}^n\sum_{i' = 1}^n\sum_{j' = 1}^n \epsilon_{i,j}\epsilon_{i',j'}k(z_{i,j},z_{i',j'})}.
\end{equation*}
In other words, there exists a positive semi-definite matrix $K\in \mathbb{R}^{n^2\times n^2}$ whose diagonal entries are no larger than 1 and a random vector $\tilde{\epsilon}\in \mathbb{R}^{n^2}$ whose entries are $\{\epsilon_{i,j}\}_{1\le i \le n,1 \le j \le n}$ such that
\begin{align*}
    \sup_{\|r\|_k \le 1}\sum_{i=1}^n\sum_{j=1}^n \epsilon_{i,j} r(z_{i,j}) = 
    \sqrt{(\tilde{\epsilon}\transpose) K\tilde{\epsilon}}.
\end{align*}
Let $K = J\transpose J$, then we have  $\|J\|_S \leq \|J\|_F = \sqrt{\text{trace}(K)} \leq n$, where $\|\cdot\|_S$ and $\|\cdot\|_F$ denote the spectral norm and Frobenious norm, respectively, and 
\begin{align*}
\sup_{\|r\|_k \le 1}\sum_{i=1}^n\sum_{j=1}^n \epsilon_{i,j} r(z_{i,j}) = \|J\tilde{\epsilon}\|_2.
\end{align*}
By the concentration of anisotropic random vectors (see, e.g.,~\cite[Theorem 6.3.2]{vershynin2018high}), we can conclude that $ \sup_{\|r\|_k \le 1}\sum_{i=1}^n\sum_{j=1}^n \epsilon_{i,j} r(z_{i,j})$, conditional on $\{z_{i,j}\}_{1\le i \le n, 1\le j \le n}$, is $Cn$-subgaussian. Hence, with probability at least $1- p$, we have
\begin{equation*}
    \sup_{\|r\|_k \le 1}\sum_{i=1}^n\sum_{j=1}^n \epsilon_{i,j} r(z_{i,j}) \le Cn  \sqrt{1+\log(1/p)}.
\end{equation*}
Therefore, with probability at least $1-p$,
\begin{equation}\label{empirical_l2_estimation}
     \frac{1}{n^2}\sum_{i=1}^n\sum_{j=1}^n[\hat{r}(z_{i,j}) - r^*(z_{i,j})]^2\le \frac{C}{n}  \sqrt{1+\log(1/p)}.
\end{equation}
Noticing $\|\hat{r} - r^*\|_k \le \|\hat{r}\|_k + \|r^*\|_k \le 2$, we next define $\bar{\nu}_n = \frac{1}{n^2}\sum_{i=1}^n\delta_{z_{i,j}}$ and estimate 
\begin{align*}
    \sup_{\|r\|_k \le 2, \|r\|_{L^2(\bar{\nu}_n)}\le C\sqrt{\frac{1}{n}[1 + \log(1/p)]}} \|r\|_{L^2(\bar{\nu})}
\end{align*}
to prove the result.
Given $\epsilon > 0$, we have
\begin{align}
    &\sup_{\|r\|_k \le 1, \|r\|_{L^2(\bar{\nu}_n)} \le \epsilon}\|r\|_{L^2(\bar{\nu})}\notag\\
    =\;&\sup_{\|r\|_k \le 1, \|r\|_{L^2(\bar{\nu}_n)} \le \epsilon}\sup_{\|g\|_{L^2(\bar{\nu})}\le 1} \int_{\bS\times\bA}g(z)r(z)\rmd\bar{\nu}(z)\notag \\
    =\;&\sup_{\|g\|_{L^2(\bar{\nu})}\le 1}\sup_{\|r\|_k \le 1, \|r\|_{L^2(\bar{\nu}_n)}\le \epsilon}\int_{\bS\times\bA}g(z)r(z)\rmd\bar{\nu}(z)\notag\\
    =\;&\sup_{\|g\|_{L^2(\bar{\nu})}\le 1}\inf_{g' \in L^2(\bar{\nu}_n)}[\mathrm{MMD}_k(g\circ \bar{\nu}, g' \circ\bar{\nu}_n) + \epsilon\|g'\|_{L^2(\bar{\nu}_n)} ]\notag\\
    =\;&\sup_{\|g\|_{L^2(\bar{\nu})}\le 1}\inf_{c_{1,1},\dots,c_{n,n}}\bigg[\mathrm{MMD}_k(g\circ \bar{\nu}, \frac{1}{n^2}\sum_{i=1}^n\sum_{j=1}^nc_{i,j}\delta_{z_{i,j}}) + \epsilon\sqrt{\frac{1}{n^2}\sum_{i=1}^n\sum_{j=1}^n c_{i,j}^2} \bigg]\notag \\
    \le\;&\sup_{\|g\|_{L^2(\bar{\nu})}\le 1}\inf_{c_{1,1},\dots,c_{n,n}}\bigg[\frac{1}{n}\sum_{j=1}^n\mathrm{MMD}_k(g\circ \nu_j, \frac{1}{n}\sum_{i=1}^nc_{i,j}\delta_{z_{i,j}}) + \epsilon\sqrt{\frac{1}{n^2}\sum_{i=1}^n\sum_{j=1}^n c_{i,j}^2}\bigg]\notag\\
    \le\;&2\bigg\{\sup_{\|g\|_{L^2(\bar{\nu})}\le 1}\inf_{c_{1,1},\dots,c_{n,n}}\bigg[\frac{1}{n}\sum_{j=1}^n\mathrm{MMD}^2_k(g\circ \nu_j, \frac{1}{n}\sum_{i=1}^nc_{i,j}\delta_{z_{i,j}}) + \frac{\epsilon^2}{n^2}\sum_{i=1}^n\sum_{j=1}^n c_{i,j}^2\bigg]\bigg\}^{\frac{1}{2}}\notag\\
    =\;&2\bigg\{\sup_{\sum_{j=1}^n m_j^2 \le n}\frac{1}{n}\sum_{j=1}^n\sup_{\|g\|_{L^2(\nu_j)}\le m_j}\inf_{c_1,\dots,c_n}\bigg[\mathrm{MMD}_k^2(g\circ \nu_j, \frac{1}{n}\sum_{i=1}^nc_i \delta_{z_{i,j}}) + \frac{\epsilon^2}{n}\sum_{i=1}^nc_i^2\bigg]\bigg\}^{\frac{1}{2}}\label{l2_estimation}.
\end{align}
In the third equality above, we have used equation \eqref{dual_rho} (when $\mathcal{B}$ is an RKHS with kernel $k$) that is proved in Lemma~\ref{thm: concentration_coefficient} below. 
For $1 \le j \le n$, we then estimate
\begin{equation*}
   \sup_{\|g\|_{L^2(\nu_j)} \le m_j} \inf_{c_1,\dots,c_n}[\mathrm{MMD}_k^2(g\circ \nu_j, \frac{1}{n}\sum_{i=1}^nc_i \delta_{z_{i,j}}) + \frac{\epsilon^2}{n}\sum_{i=1}^nc_i^2].
\end{equation*}
We first find a probability distribution $\lambda$ on $\mathbb{N}^+$ and a measurable mapping $\phi: \mathbb{N}^+\times (\bS\times\bA)$ such that
\begin{equation*}
    k(z,z') = \EE_{\omega \sim \lambda}\phi(\omega,z)\phi(\omega,z'),
\end{equation*}
which can be achieved using Mercer decomposition (see, e.g., \cite[Section~2.3]{bach2017equivalence}). Then,
\begin{equation*}
    \mathrm{MMD}_k^2(g\circ \nu_j, \frac{1}{n}\sum_{i=1}^nc_i \delta_{z_{i,j}}) = \EE_{\omega\sim \lambda}[\int_{\bS\times\bA}g(z)\phi(\omega,z) \rmd \nu_j(z) - \frac{1}{n}\sum_{i=1}^nc_i \phi(\omega,z_{i,j})]^2.
\end{equation*}
Recalling that $z_{1,j},\dots,z_{n,j}$ are i.i.d. drawn from $\nu_j$, we can use Proposition 1 in \cite{bach2017equivalence} to obtain that with probability at least $1-p$
\begin{equation*}
    \sup_{\|g\|_{L^2(\nu_j)}\le 1}\inf_{\sum_{i=1}^nc_i^2 \le 4n}\EE_{\omega\sim \lambda}[\int_{\bS\times\bA}g(z)\phi(\omega,z) \rmd \nu_j(z) - \frac{1}{n}\sum_{i=1}^nc_i \phi(\omega,z_{i,j})]^2 \le 4t,
\end{equation*}
where $t$ satisfies that 
\begin{align}
5d(t)\log(\frac{16d(t)}{p})  = n,
\label{dt_estimate1}
\end{align}
and \begin{align}
    &d(t) = \sup_{z \in \bS \times\bA}\langle \phi(\cdot,z), (\Sigma + t I)^{-1}\phi(\cdot,z)\rangle_{L^2(\lambda)} \notag \\ 
    \le~& t^{-1}\sup_{z \in \bS \times \bA}\langle \phi(\cdot,z), \phi(\cdot,z)\rangle_{L^2(\lambda)} \notag \\
    =~&t^{-1}\sup_{z \in \bS \times \bA}k(z,z) \le t^{-1},
\label{dt_estimate2}
\end{align}
in which $\Sigma$ is a self-adjoint, positive semi-definite operator on $L^2(\lambda)$ (see the detailed definition in~\cite[Section~2.1]{bach2017equivalence}).
Therefore, from \eqref{dt_estimate1}\eqref{dt_estimate2}, we have
\begin{align*}
    &\sup_{\|g\|_{L^2(\nu_j)}\le 1}\inf_{\sum_{i=1}^nc_i^2 \le 4n}\EE_{\omega\sim \lambda}[\int_{\bS\times\bA}g(z)\phi(\omega,z) \rmd \nu_j(z) - \frac{1}{n}\sum_{i=1}^nc_i \phi(\omega,z_{i,j})]^2 \\
    \le~& \frac{4}{d(t)} \le \frac{C}{n}[1 + \log(\frac{n}{p})],
\end{align*}
which means that
\begin{equation}
     \sup_{\|g\|_{L^2(\nu_j)} \le m_j} \inf_{c_1,\dots,c_n}[\mathrm{MMD}_k^2(g\circ \nu_j, \frac{1}{n}\sum_{i=1}^nc_i \delta_{z_{i,j}}) + \frac{\epsilon^2}{n}\sum_{i=1}^nc_i^2] \le Cm_j^2 \frac{1}{n}[1 + \log(\frac{n}{p})] + Cm_j^2 \epsilon^2.
\end{equation}
Combining the last inequality with \eqref{l2_estimation}, we can obtain that
\begin{equation}
\label{population_l2_bound}
    \sup_{\|r\|_k \le 1, \|r\|_{L^2(\bar{\nu}_n)} \le \epsilon}\|r\|_{L^2(\bar{\nu})} \le C\sqrt{\frac{1}{n}[1 + \log(\frac{n}{p})]} + C\epsilon.
\end{equation}
Recalling inequality \eqref{empirical_l2_estimation} and choosing $\epsilon = C\sqrt{\frac{1}{n}[1 + \log(1/p)]}$ in~\eqref{population_l2_bound}, we know that with probability at least $1-p$
\begin{align*}
    \sup_{\|r\|_k \le 2, \|r\|_{L^2(\bar{\nu}_n)}\le C\sqrt{\frac{1}{n}[1 + \log(1/p)]}} \|r\|_{L^2(\bar{\nu})} \le C\sqrt{\frac{1}{n}[1 + \log(\frac{n}{p})]},
\end{align*}
which concludes the proof.
\end{proof}
\begin{proof}[Proof of Theorem~\ref{upper_bound_known}]
Using Lemma \ref{lem_1} and the union bound, we know that for any $\theta \in \Theta$, with probability at least $1-p$, 
\begin{equation*}
    \|\hat{r}_\theta(h,\cdot) - r_\theta(h,\cdot)\|_{L^2(\hat{\nu})} \le C\sqrt{\frac{1}{n}[1 + \log(\frac{nH}{p})]},\; \forall h \in [H].
\end{equation*}
Recalling the definition of $\Delta_\mathcal{M}(n^{-\frac{1}{2}})$ in Definition~\ref{def:delta1} in the case of known transition
\begin{equation*}
    \Delta_\mathcal{M}(n^{-\frac{1}{2}}) = \sup_{h \in [H]} \sup_{\|r\|_k \le 1, \sqrt{n}\|r\|_{L^2(\hat{\nu})} \le 1}\|r\|_{\Pi(h,P_0,\mu)}, 
\end{equation*}
we know that for any $\pi \in \mathcal{P}(\bA \cond \bS,H)$ and $h \in [H]$,
\begin{align*}
    &\bigg|\int_{\bS\times\bA}r_\theta(h,s,a)\rmd\rho_{h,P_0,\pi,\mu}(s,a) - \int_{\bS\times\bA} \hat{r}_\theta(s,a)\rmd\rho_{h,P_0,\pi,\mu}(h,s,a)\bigg| \\
    \le~& C\Delta_\mathcal{M}(n^{-\frac{1}{2}}) \sqrt{1+ \log(\frac{nH}{p})},
\end{align*}
which means that for any $\pi \in \mathcal{P}(\bA\cond\bS,H)$
\begin{equation*}
    |J(M_\theta,\pi) - J(\hat{M}_{\theta},\pi)| \le CH\Delta_\mathcal{M}(n^{-\frac{1}{2}})\sqrt{1+ \log(\frac{nH}{p})}.
\end{equation*}
Therefore,
\begin{align*}
   0 &\le  J^*(M_\theta) - J(M_\theta,\hat{\pi}_\theta) = J^*(M_\theta) - J^*(\hat{M}_\theta) + J^*(\hat{M}_\theta) - J(M_\theta,\hat{\pi}_\theta)\\
   &= \sup_{\pi \in \mathcal{P}(\bA \cond \bS)} J(M_\theta,\pi) - \sup_{\pi \in \mathcal{P}(\bA \cond \bS)} J(\hat{M}_\theta,\pi) + J(\hat{M}_\theta,\hat{\pi}_\theta) - J(M_\theta,\hat{\pi}_\theta) \\
   &\le  CH\Delta_\mathcal{M}(n^{-\frac{1}{2}})\sqrt{1+ \log(\frac{nH}{p})}.
\end{align*}
\end{proof}
\begin{rem}\label{equivalence_2}
We remark that equivalence exists between the RL problem with known transition and the supervised learning problem with respect to the $\Pi$-norm in the RKHS.
On the one hand, the above proof shows that, for any function $g$ in the unit ball of the RKHS and any $h \in [H]$, once we can use finite samples to obtain an estimation $\hat{g}$  which is accurate with respect to $\|\cdot\|_{\Pi(h,P_0,\mu)}$, then the corresponding RL problem can be solved efficiently.
On the other hand, Lemma \ref{lem_1} shows that one can always use finite samples to efficiently obtain an accurate estimation, in the $L^2$ sense, of a target function lying in the unit ball of the RKHS, and Theorem \ref{thm: Known Transition} tells that once the RL problem can be solved efficiently, $\Delta_\mathcal{M}(\epsilon)$ must decay fast with respect to $\epsilon$. 
Therefore, we conclude that if the RL problem can be solved efficiently, for any $h \in [H]$, one can always obtain an estimation $\hat{g}$ of a target function $g$ in the unit ball of the RKHS which is accurate with respect to $\|\cdot\|_{\Pi(h,P_0,\mu)}$.
\end{rem}

\subsection{The Case of Unknown Transition}
\label{sec:upper_unknown}
Similar to the case of known transition in the previous subsection, we consider
\begin{equation*}
    \hat{\xi} = \argmin_{\bar{\xi} \in\bar{\Xi}_n} \sup_{\theta \in \Theta}\mathcal{R}(\Pi(P_\theta,\mu),\mathcal{H}_k,n^{-\frac{1}{2}},\nu^{\theta,\bar{\xi}}),
\end{equation*}
an intrinsic property of the MDP family $\mathcal{M}$ in the case of unknown transition. Again, similar argument also holds if $\hat{\xi}$ is not the minimizer but
\begin{equation*}
  \sup_{\theta \in \Theta}\mathcal{R}(\Pi(P_\theta,\mu),\mathcal{H}_k,n^{-\frac{1}{2}},\nu^{\theta,\hat{\xi}})
\end{equation*}
is small.
Given $\hat{\xi}$, we sample $$(\hat{z}_{1,1}^{\theta},\dots,\hat{z}_{1,n}^{\theta}),\dots,(\hat{z}_{n,1}^\theta,\dots,\hat{z}_{n,n}^\theta)$$ as i.i.d. copies of $((s_1^{\theta,\hat{\xi}},a_1^{\theta,\hat{\xi}}),\dots,(s_n^{\theta,\hat{\xi}},a_n^{\theta,\hat{\xi}}))$ defined in \eqref{def_sample_path_2}. We then use $\{\hat{z}_{i,j}\}_{1\le i \le n, 1\le j \le n}$ to perform the following variant of fitted Q-iteration algorithm (see, e.g.,~\cite{chen2019information,fan2020theoretical,long20212}), as summarized in Algorithm~\ref{alg:FittedQ}.
To understand the idea of Algorithm~\ref{alg:FittedQ}/fitted Q-iteration, we define the optimal action-value function (Q-value function) $Q_h^*:\bS \times \bA \mapsto \bR$ as the optimal expected cumulative reward of the MDP starting from step $h$:
\begin{align}\label{def_optimal_q}
    Q_h^{\theta,*}(s,a) &=\sup_{\pi \in \mathcal{P}(\bA \cond \bS,H)}\EE_{P_\theta,\pi}[\sum_{h'= h}^H r_\theta(h', S_{h'},A_{h'}) \cond S_h = s, A_h = a].
\end{align}
Here the expectation is taken among the MDP paths generated by transition probability $P_\theta$ and policy $\pi$. 
Let $\pi^{\theta,*} = (\pi^{\theta,*}_h)_{h \in [H]}$ be the greedy policy with respect to $Q^{\theta,*} = (Q_h^{\theta,*})_{h\in[H]}$:
\begin{equation*}
    \mathrm{supp}(\pi^{\theta,*}_ h(\,\cdot\cond s)) \subset \{a \in \bA: Q_h^\theta(s,a) = \max_{a'\in\bA}Q_h^\theta(s,a')\},
\end{equation*}
for any $s \in \bS$ and $ h \in [H]$. Here $\mathrm{supp}(\pi^{\theta,*}_ h(\,\cdot\cond s))$ denotes the support of $\pi^{\theta,*}_h(\,\cdot\cond s)$. We can then conclude that $\pi^{\theta,*}$ is the optimal policy of $M_\theta$ (see, e.g. \cite[Theorem~4.5.1]{puterman2014markov}):
\begin{equation*}
    J(M_\theta,\pi^{\theta,*}) = \sup_{\pi \in \mathcal{P}(\bA \cond \bS,H)} J(M_\theta,\pi).
\end{equation*}

We define the Bellman optimal operator $\mathcal{T}_h^\theta: \mathcal{H}_k \mapsto \mathcal{H}_k$ as follows:
\begin{equation}
\label{Bellman_optimal_operator}
    (\mathcal{T}_h^\theta g)(s,a) = r_\theta(h,s,a) + \EE_{s'\sim P_\theta(\,\cdot\,|\,h,s,a)}[\max_{a' \in \bA}g(s',a')].
\end{equation}
Then the famous Bellman equations gives 
\begin{equation}\label{Bellman_optimal}
    Q_{h}^{\theta,*} = \mathcal{T}_h^{\theta} Q_{h+1}^{\theta,*},\, \forall h \in [H], \,Q_{H+1}^{\theta,*} = 0.
\end{equation}
We assume that for any $g \in \mathcal{H}_k$ and $\theta \in \Theta$,
\begin{equation}\label{Bellman_assumption}
    \|\mathcal{T}_h^\theta g\|_k \le \|g\|_k + 1,
\end{equation}
which
 implies that
\begin{equation}
\label{eq:Q_bound}
\|Q_h^{\theta,*}\|_k \le H-h + 1.
\end{equation}
Finally equations~\eqref{Bellman_optimal}\eqref{eq:Q_bound} motivate us to solve the optimization problem~\eqref{optimization_problem2} in Algorithm~\ref{alg:FittedQ}. Note that~\eqref{optimization_problem2} can be solved as a finite-dimensional convex optimization problem in the same way as~\eqref{optimization_problem1}; see the comment before Algorithm~\ref{alg:FittedReward}.

\begin{algorithm}[ht]
\caption{Fitted Q-Iteration Algorithm}
{
\KwIn{$n^2$ samples $(\hat{z}_{1,1}^{\theta}, \dots, \hat{z}_{1,n}^{\theta}), \dots, (\hat{z}_{n,1}^\theta, \dots, \hat{z}_{n,n}^\theta)$ as i.i.d. copies of $((s_1^{\theta,\hat{\xi}},a_1^{\theta,\hat{\xi}}), \dots, (s_n^{\theta,\hat{\xi}}, a_n^{\theta,\hat{\xi}}))$ defined in \eqref{def_sample_path_2}.}
\textbf{Initialize:} $Q^\theta_{H+1}(s,a) = 0$ for any $(s,a) \in \bS \times \bA$.}

\For{$h = H,H-1,\dots,1$}{
 
\For{$i = 1,\dots,n$ and $j = 1,\dots,n$}{
Sample $r_{i,j}^\theta\sim \mathcal{N}(r_\theta(h,\hat{z}_{i,j}^\theta),1)$ and $s^{\theta,'}_{i,j}\sim P_\theta(\,\cdot\,|\,h,\hat{z}_{i,j}^\theta)$\\
 Compute $y_{i,j}^\theta = r_{i,j}^\theta + \max_{a' \in \bA} Q_{h+1}(s_{i,j}^{\theta,'},a')$\\
 }
Compute $Q_h^\theta$ as the minimizer of the optimization problem
\begin{equation}\label{optimization_problem2}
   \min_{\|g\|_k \le H-h+1}\sum_{i= 1}^n\sum_{j=1}^n[g(\hat{z}_{i,j}^\theta) - y_{i,j}^\theta]^2\;
\end{equation}
}
\KwOut{$\hat{\pi}^\theta$ as the greedy policies with respect to $(Q_h^\theta)_{h\in[H]}$.

}
\label{alg:FittedQ}
\end{algorithm}

We have the following convergence result regarding Algorithm~\ref{alg:FittedQ}.

\begin{thm}\label{upper_bound_unknown}
Assume that \eqref{Bellman_assumption}  holds and 
\begin{equation*}
    \sup_{z\in\bS\times\bA}k(z,z) \le 1.
\end{equation*}
Then, for any $\theta \in \Theta$, with probability at least $1-p$,
\begin{equation*}
    |J(M_\theta,\hat{\pi}^\theta) - J^*(M_\theta)|\le CH^3\Delta_\mathcal{M}(n^{-\frac{1}{2}})\sqrt{1+ \log(\frac{nH}{p})}.
\end{equation*}
\end{thm}
\begin{rem}
Assumption \eqref{Bellman_assumption} is used to control the approximation error in fitted Q-iteration. Similar assumptions are made in \cite{chen2019information,fan2020theoretical,wang2020reinforcement,yang2020provably,yang2020function,long20212}. By choosing $g = 0$ in \eqref{Bellman_assumption}, we can see that \eqref{Bellman_assumption} is stronger than the setting in Theorem \ref{thm: unknown_transition}, which only requires $\|r_\theta(h,\cdot)\|_k \le 1$ for any $\theta \in \Theta$ and $h \in [H]$. How to fill out this gap is left for future work.
\end{rem}

\begin{proof}
For each time step $h \in [H]$, by equation \eqref{optimization_problem2}, we have $\|Q_h^\theta\|_k \le H-h + 1$.
With the optimal Bellman operator \eqref{Bellman_optimal_operator}, we have
\begin{align*}
    y_{i,j}^\theta & = r_{\theta}(h,\hat{z}^\theta_{i,j})+\epsilon_{i,j} + \max_{a'\in\bA}Q_{h+1}^\theta(s_{i,j}^{\theta,'},a')\\
    & = (\mathcal{T}_h^\theta Q_{h+1}^\theta)(\hat{z}^\theta_{i,j})+ (\max_{a'\in\bA}Q_{h+1}^\theta(s_{i,j}^{\theta,'},a')- \EE_{s' \sim P_\theta(\,\cdot\cond h,\hat{z}_{i,j}^\theta)}[\max_{a'\in \bA}Q^\theta_{h+1}(s_{i,j}^{\theta,'},a')])   +\epsilon_{i,j},
\end{align*}
where $\epsilon_{i,j}$ are standard normal random variables.
With the boundedness of $Q$
\begin{align*}
    \sup_{z \in \bS\times\bA}|Q_{h+1}^\theta(z)| &= \sup_{z \in \bS\times\bA}|\langle Q_{h+1}^\theta, k(\cdot,z)\rangle_k| \\
    &\le (H-h) \sup_{z \in \bS\times\bA}\|k(\cdot,z)\|_k = (H-h)\sup_{z \in \bS\times\bA}\sqrt{k(z,z)} \le H-h
\end{align*}
we know that that $y_{i,j}^\theta - (\mathcal{T}_h^\theta Q^\theta_{h+1})(\hat{z}^\theta_{i,j})$ is $CH$-subgaussian; see, e.g., \cite[Example 2.5.8]{vershynin2018high}.
With the assumption \eqref{Bellman_assumption}, we know $\|\mathcal{T}_h^\theta Q^\theta_{h+1}\|_k \le H-h + 1$. Hence, we can use Lemma~\ref{lem_1} and the union bound to obtain that with probability at least $1-p$, for any $h \in [H]$,
\begin{equation*}
    \|Q_h^\theta - \mathcal{T}_h^\theta Q_{h+1}^\theta\|_{L^2(\nu^{\theta,\hat{\xi}})}\le CH\sqrt{\frac{1}{n}[1 + \log(\frac{nH}{p})]}.
\end{equation*}
Noticing that $\|Q_h^\theta - \mathcal{T}_h^\theta Q^\theta_{h+1}\|_k \le 2H$ and the definition of $\Delta_\mathcal{M}(n^{-\frac{1}{2}})$ in \eqref{eq:definition_delta2}, we have that, with probability at least $1-p$, for any $h \in [H]$,
\begin{align}
    &\sup_{\pi \in \mathcal{P}(\bA \cond \bS,H)}|\int_{\bS\times\bA}[Q_h^\theta(s,a) - (\mathcal{T}_h^\theta Q_{h+1}^\theta)(s,a)]\rmd \rho_{h,P_\theta,\pi,\mu}(s,a)|\\
    \le&~ CH\Delta_\mathcal{M}(n^{-\frac{1}{2}})\sqrt{1 + \log(\frac{nH}{p})}.
    \label{thm4_eq1}
\end{align}

Using the optimal Bellman equation \eqref{Bellman_optimal}, we have:
\begin{align}
       &\bigg|\int_{\bS\times\bA}[Q_h^\theta - Q_h^{\theta,*}](s,a)\rmd \rho_{h,P_\theta,\pi,\mu}(s,a)\bigg| \notag \\
    \le~&\bigg|\int_{\bS\times\bA}[Q_h^\theta -(\mathcal{T}_h^\theta Q_{h+1}^\theta)](s,a) \rmd \rho_{h,P_\theta,\pi,\mu}(s,a)\bigg| \notag \\
    &\, + \bigg|\int_{\bS\times\bA}[(\mathcal{T}_h^\theta Q_{h+1}^{\theta,*}) -(\mathcal{T}_h^\theta Q_{h+1}^\theta) ](s,a)\rmd \rho_{h,P_\theta,\pi,\mu}(s,a)\bigg|.\label{thm4_eq2}
\end{align}
Notice that
\begin{align*}
    &\bigg|\int_{\bS\times\bA}[(\mathcal{T}_h^\theta Q_{h+1}^{\theta,*}) -(\mathcal{T}_h^\theta Q_{h+1}^\theta) ](s,a)\rmd \rho_{h,P_\theta,\pi,\mu}(s,a)\bigg| \\
    \le\; &\bigg|\int_{\bS\times\bA}\EE_{s'\sim P_\theta(\,\cdot\cond h,s,a)}[\max_{a' \in \bA}Q_{h+1}^{\theta,*} - \max_{a' \in \bA}Q_{h+1}^\theta](s',a')\rmd \rho_{h,P_\theta,\pi,\mu}(s,a)\bigg| \\
    \le\;&\max\bigg\{\bigg|\int_{\bS\times\bA}\EE_{s'\sim P_\theta(\,\cdot\cond h,s,a)}[\max_{a' \in \bA}[Q_{h+1}^{\theta,*}-Q_{h+1}^\theta](s',a') \rmd \rho_{h,P_\theta,\pi,\mu}(s,a)\bigg|,\\
    &\;\qquad\;\;\,\bigg|\int_{\bS\times\bA}\EE_{s'\sim P_\theta(\,\cdot\cond h,s,a)}[\max_{a' \in \bA}[Q_{h+1}^\theta-Q^{\theta,*}_{h+1}](s',a') \rmd \rho_{h,P_\theta,\pi,\mu}(s,a)\bigg|\bigg\}\\
    \le\;&   \sup_{\pi \in\mathcal{P}(\bA \cond \bS,H) }\bigg|\int_{\bS\times\bA}[Q_{h+1}^\theta - Q_{h+1}^{\theta,*}](s,a) \rmd \rho_{h+1,P_\theta,\pi,\mu}(s,a)\bigg|,
\end{align*}
where the last inequality holds because we can choose policies $\pi^1$ and $\pi^2$ such taht $\pi^1_{h'} = \pi^2_{h'} = \pi$ for $h' \neq h + 1$ and $\pi^1_{h+1}$ and $\pi^2_{h+1}$ are the greedy policies of $[Q_{h+1}^{\theta,*} - Q_{h+1}^\theta](s,a)$ and $[Q_{h+1}^\theta - Q_{h+1}^{\theta,*}](s,a)$, respectively. Combining the last inequality, inequalities \eqref{thm4_eq1} and \eqref{thm4_eq2}, we have that with probability $1-p$, $\forall h \in [H]$
\begin{align*}
    &\sup_{\pi \in \mathcal{P}(\bA|\bS,H)} \bigg|\int_{\bS\times\bA}[Q_h^\theta - Q_h^{\theta,*}](s,a)\rmd \rho_{h,P_\theta,\pi,\mu}(s,a)\bigg| \\
    \le &CH\Delta_\mathcal{M}(n^{-\frac{1}{2}})\sqrt{1 + \log(\frac{nH}{p})} + \sup_{\pi \in\mathcal{P}(\bA \cond \bS,H) }\bigg|\int_{\bS\times\bA}[Q_{h+1}^\theta - Q_{h+1}^{\theta,*}](s,a) \rmd \rho_{h+1,P_\theta,\pi,\mu}(s,a)\bigg|.
\end{align*}
With the recursive relationship above, we have that, with probability at least $1-p$,
\begin{equation*}
    \sup_{h \in [H]}\sup_{\pi \in \mathcal{P}(\bA \cond \bS,H)}\bigg|\int_{\bS\times\bA}[Q_h^\theta - Q_h^{\theta,*}](s,a)\rmd \rho_{h,P_\theta,\pi,\mu}(s,a)\bigg| \le CH^2\Delta_\mathcal{M}(n^{-\frac{1}{2}})\sqrt{1+ \log(\frac{nH}{p})}.
\end{equation*}
We can then use the famous performance difference lemma (see, e.g., \cite[Lemma~3.2]{cai2020provably} or \cite[Lemma~6.1]{kakade2002approximately}) to obtain that
\begin{align}\label{performance_dif}
    0 &\le J^*(M_\theta) - J(M_\theta,\hat{\pi}^\theta) = \sum_{h=1}^H\int_{\bS\times\bA}\sum_{a' \in \bA}Q_h^{\theta,*}(s,a')[\pi_h^{\theta,*}(a'\cond s)-\hat{\pi}^{\theta}_h(a' \cond s)]\rmd \notag \rho_{h,P_\theta,\hat{\pi}^\theta,\mu}(s,a) \\
    &= \sum_{h=1}^H\int_{\bS\times\bA}\sum_{a' \in \bA}\{[Q_h^{\theta,*}-Q_h^\theta](s,a')\pi_h^{\theta,*}(a'\cond s)+Q_h^\theta(s,a')[\pi^{\theta,*}_h - \hat{\pi}^{\theta}_h](a'\cond s)\\
    &\qquad \qquad \qquad\qquad + [Q_h^\theta-Q_h^{\theta,*}](s,a')\hat{\pi}^{\theta}_h(a'\cond s)\}
    \rmd \rho_{h,P_\theta,\hat{\pi}^\theta,\mu}(s,a).
\end{align}
Noticing that $\sum_{a' \in \bA}Q_h^\theta(s,a')[\pi_h^{\theta,*} - \hat{\pi}^{\theta}_h](a'|s) \le 0$ since $\hat{\pi}^\theta$ is the greedy policy with respect to $Q_h^\theta$, we can conclude that, with probability at least $1-p$,
\begin{align*}
    |J^*(M_\theta) - J(M_\theta,\hat{\pi}^\theta)| &\le CH\sup_{h \in [H]}\sup_{\pi \in \mathcal{P}(\bA \cond \bS,H)}|\int_{\bS\times\bA}[Q_h^\theta - Q_h^{\theta,*}](s,a)\rmd \rho_{h,P_\theta,\pi,\mu}(s,a)|\\ &\le CH^3\Delta_\mathcal{M}(n^{-\frac{1}{2}})\sqrt{1+ \log(\frac{nH}{p})}.
\end{align*}
\end{proof}

\section{Discussion on Perturbational Complexity by Distribution Mismatch}\label{Sec_concentra}
In this section, we discuss in more details the perturbation response  $\mathcal{R}(\Pi,\cB,\epsilon,\nu)$ and the perturbational complexity $\Delta_{\mathcal{M}}(\epsilon)$. We first give a more concrete expression of $\mathcal{R}(\Pi,\cB,\epsilon,\nu)$.

\begin{prop}\label{thm: concentration_coefficient}
We have
\begin{equation}
   \mathcal{R}(\Pi,\cB,\epsilon,\nu) = \sup_{\rho \in \Pi}\inf_{g \in L^2(\nu)}[\|\rho - g\circ \nu\|_{\mathcal{B}^{*}} + \epsilon\|g\|_{L^2(\nu)}],
\end{equation}
where $g\circ \nu$ is a signed measure such that 
\begin{equation*}
    \frac{\rmd g \circ \nu}{\rmd \nu} = g,
\end{equation*}
$\mathcal{B}^*$ is the dual space of $\mathcal{B}$
and $\|\rho\|_{\mathcal{B}^{*}}$ is the dual norm of linear functional 
\begin{equation*}
    \rho(g) \coloneqq \int_{\bS\times\bA}g(z)\rmd \rho(z),\,\forall g \in \mathcal{B},
\end{equation*}
for any signed measure $\rho$ on $\bS\times\bA$ (we slightly abuse the notation that $\rho$ are both the signed measure and linear functional in $\mathcal{B})$.
If $\mathcal{B}$ is an RKHS with kernel k, then
\begin{align}\label{kernel_dual}
    \mathcal{R}(\Pi,\cH_k,\epsilon,\nu) = \sup_{\rho \in \Pi}\inf_{g \in L^2(\nu)}
   [\mathrm{MMD}_k(\rho, g\circ \nu) + \epsilon\|g\|_{L^2(\nu)}].
\end{align} 

\end{prop}
\begin{proof}
It is sufficient to prove that
\begin{equation}\label{dual_rho}
    \sup_{\|g\|_\mathcal{B} \le 1, \|g\|_{L^2(\nu)}\le \epsilon }\int_{\bS\times\bA} g(z)\rmd \rho(z) = \inf_{g \in L^2(\nu)}[\|\rho - g\circ\nu\|_{\mathcal{B}^{*}} + \epsilon\|g\|_{L^2(\nu)}].
\end{equation}
Equation \eqref{kernel_dual} can be derived using the definition of maximum mean discrepancy.

Define $F_1, F_2\colon \mathcal{B}\mapsto (-\infty, +\infty]$:
\begin{equation*}
    F_1(g) = \begin{cases}
    \int_{\bS\times\bA}g(z)\rmd\rho(z) &\text{ if }\|g\|_\mathcal{B}\le 1 \\
    +\infty &\text{ if } \|g\|_\mathcal{B} > 1
    \end{cases},\,
    F_2(g) = \begin{cases}
   0 &\text{ if } \|g\|_{L^2(\nu)}\le \epsilon\\
    +\infty &\text{ if } \|g\|_{L^2(\nu)}> \epsilon
    \end{cases}.
\end{equation*}
Then $F_1$ and $F_2$ are both convex functions in $\mathcal{B}$. We can define the conjugate functions $F_1^*$ and $F_2^*\colon  \mathcal{B}^*\mapsto(-\infty,+\infty]$:
\begin{align*}
    &F_1^*(l) = \sup_{g \in \mathcal{B}}[l(g) - F_1(g)] = \sup_{\|g\|_\mathcal{B}\le 1}[l(g) - \rho(g)] = \|l - \rho\|_{\mathcal{B}^*},\\
    &F_2^*(l) = \sup_{g \in \mathcal{B}}[l(g) - F_2(g)] = \sup_{g \in \mathcal{B},\|g\|_{L^2(\nu)} \le \epsilon}l(g) = \epsilon\sup_{g \in \mathcal{B},\|g\|_{L^2(\nu)}\le 1} l(g). 
\end{align*}

We then compute
\begin{equation*}
    \sup_{g \in \mathcal{B},\|g\|_{L^2(\nu)}\le 1} l(g).
\end{equation*}
Let $l \in \mathcal{B}^*$ such that
\begin{equation*}
     M_l = \sup_{g \in \mathcal{B},\|g\|_{L^2(\nu)}\le 1} l(g) < +\infty.
\end{equation*}
Notice that $l$ is a linear mapping in $\mathcal{B}$ such that
\begin{equation*}
    |l(g)| \le M_l \|g\|_{L^2(\nu)},\; \forall \, g \in \mathcal{B}.
\end{equation*}
Using Hahn-Banach theorem \cite[Corollary~1.2]{brezis2010functional} and Riesz representation theorem in $L^2(\nu)$ \cite[Theorem~4.11]{brezis2010functional},  we know that there exists $R_l \in L^2(\nu)$ such that
\begin{equation}\label{l_g_relation}
    l(g) = \int_{\bS\times\bA}R_l(z)g(z)\rmd \nu(z),\; \forall\, g \in \mathcal{B}.
\end{equation}
Hence,
\begin{equation*}
     \sup_{g \in \mathcal{B},\|g\|_{L^2(\nu)}\le 1} l(g) = \|P_\mathcal{B} R_l\|_{L^2(\nu)},
\end{equation*}
where $P_\mathcal{B}$ is the orthogonal projection from $L^2(\nu)$ to $\mathcal{B}$. Consequently,
\begin{equation*}
    F_2^*(l) = \begin{cases}
        \epsilon\|P_\mathcal{B} R_l\|_{L^2(\nu)} &\text{ if exists } R_l \in L^2(\nu) \text{ s.t.  equation \eqref{l_g_relation} holds} \\
        +\infty &\text{ otherwise}
    \end{cases}.
\end{equation*}

Noticing that $F_1$ is continuous at 0, we can use Fenchel-Rockafellar Theorem (\cite[Theorem~1.12]{brezis2010functional}) to obtain that
\begin{equation*}
    \inf_{r \in \mathcal{B}}[F_1(r) + F_2(r)] = -\inf_{l \in \mathcal{B}^*}[F_1^*(l) + F_2^*(-l)],
\end{equation*}
which means that
\begin{equation*}
    \sup_{\|g\|_\mathcal{B} \le 1,\|g\|_{L^2(\nu)}\le \epsilon}\int_{\bS\times\bA}g(z)\rmd\rho(z) = \inf_{g \in L^2(\nu)}[\|\rho - g\circ\nu\|_{\mathcal{B}^*} + \epsilon\|P_\mathcal{B}g\|_{L^2(\nu)}].
\end{equation*}
Finally, noticing that $\|\rho - g \circ \nu\|_{\mathcal{B}^{*}} = \|\rho - (P_\mathcal{B} g) \circ \nu\|_{\mathcal{B}^{*}}$, we obtain
\begin{align*}
    \inf_{g \in L^2(\nu)}[\|\rho - g\circ\nu\|_{\mathcal{B}^*} + \epsilon\|P_\mathcal{B}g\|_{L^2(\nu)}] &= \inf_{g \in L^2(\nu)}[\|\rho - (P_\mathcal{B} g)\circ\nu\|_{\mathcal{B}^*} + \epsilon\|P_\mathcal{B}g\|_{L^2(\nu)}] \\
    &= \inf_{g \in L^2(\nu)}[\|\rho - g\circ\nu\|_{\mathcal{B}^*} + \epsilon\|g\|_{L^2(\nu)}],
\end{align*}
which completes the proof.
\end{proof}

In the following, we again only consider the case that $\mathcal{B}$ is an RKHS with kernel $k$. 
We first show that the finite concentration coefficients, considered in \cite{munos2008finite,farahmand2010error,scherrer2015approximate,farahmand2016regularized,chen2019information,fan2020theoretical,agarwal2021theory,long20212},  implies that the perturbation response must decay fast. So our results generalize the previous works based on the assumption of concentratability. Note that the original assumption on concentration coefficients is only stated for the case $p = 2$, and the corresponding $M$ is called concentration coefficients in previous works.

\begin{prop}\label{concentrated_case}
Assume that there exists $1 < p \le 2$ and a distribution $\nu$ such that
\begin{equation}
M = \sup_{\rho \in \Pi}\|\frac{\rmd \rho}{\rmd\nu}\|_{L^p(\nu)} < +\infty,
\label{prop3_condition}
\end{equation}
and
\begin{equation*}
    \sup_{z \in \bS\times\bA}k(z,z)\le 1.
\end{equation*}
Then,
\begin{equation}
    \mathcal{R}(\Pi,\cH_k,n^{-\frac{1}{2}},\nu)\le 2M n^{\frac{1}{p}-1}.
\end{equation}
\end{prop}
\begin{proof} 
For any $\rho \in \Pi$, define
\begin{equation*}
    g_\rho = \frac{\rmd \rho}{\rmd \nu},
\end{equation*}
and choose $K > 0$, then
\begin{align*}
    \inf_{g \in L^2(\nu)}[\mathrm{MMD}_k(\rho, g\circ \nu) + \frac{\|g\|_{L^2(\nu)}}{\sqrt{n}}] &\le \mathrm{MMD}_k(g_\rho \circ \nu, (g_\rho\mathrm{1}_{|g_\rho|\le K})\circ \nu)) + \frac{\|g_\rho \mathrm{1}_{|g_\rho|\le K}\|_{L^2(\nu)}}{\sqrt{n}} \notag\\
    &\le \|g_\rho \mathrm{1}_{|g_\rho| > K}\|_{L^1(\nu)} + \frac{\|g_\rho \mathrm{1}_{|g_\rho|\le K}\|_{L^2(\nu)}}{\sqrt{n}}\notag\\
    &= \|g_\rho^p \cdot g_\rho^{1-p}\mathrm{1}_{|g_\rho| > K}\|_{L^1(\nu)} + \frac{\sqrt{\|g_\rho^p \cdot g_\rho^{2-p}\mathrm{1}_{|g_\rho|\le K}\|_{L^1(\nu)}}}{\sqrt{n}}\notag\\
    &\le M^p K^{1-p} + \frac{M^{\frac{p}{2}} K^{1-\frac{p}{2}}}{\sqrt{n}}.
\end{align*}
By choosing $K = M n^{\frac{1}{p}}$, we know that
\begin{equation}
    \mathcal{R}(\Pi,\cH_k,n^{-\frac{1}{2}},\nu)\le 2M n^{\frac{1}{p}-1}.
\end{equation}
\end{proof}

Next, we give a sufficient condition such that $\Delta_\mathcal{M}(n^{-\frac{1}{2}})$ decays fast with respect to $n$ in the case of unknown transition. The original idea is from \cite[Section~3.3]{long20212}. The condition is that there exists a distribution $\lambda$ on $[H]\times\bS\times\bA$ such that the induced state-action distribution
\begin{align*}
    &\bar{\nu}^\theta(S\times A) = \int_{[H]\times\bS\times\bA}P_\theta(\,S\cond h,s,a)\rmd\lambda(h,s,a) \times \mathrm{Uniform}_{\bA}(A),\\ &\forall \text{\,measurable set } S \subset \bS, A \subset \bA
\end{align*}
satisfying that 
\begin{equation*}
    \sup_{\theta\in\Theta}\mathcal{R}(\Pi(P_\theta,\mu),\cB,n^{-\frac{1}{2}},\bar{\nu}^\theta)
\end{equation*}
decays fast with respect $n$. This condition holds when
\begin{equation*}
    \sup_{\theta \in \Theta}\sup_{\rho \in \Pi(P_\theta,\mu)}\|\frac{\mathrm{d} \rho\phantom{^\theta}}{\mathrm{d} \bar{\nu}^\theta}\|_{L^2(\bar{\nu}^\theta)} < +\infty,
\end{equation*}
or when the eigenvalue decay of the kernel is fast; see the discussion below.
Given this condition, we can choose a sampling algorithm $\bar{\xi} = (\tilde{f}_1,\dots,\tilde{f}_{2n})$ in \eqref{def_sample_path_2} satisfying that
\begin{equation*}
\tilde{f}_i(\mathcal{D}_{i-1}^{\theta,\bar{\xi}},\bar{u}) = \begin{cases}
    f_i(\bar{u})\, \text{ when } 1 \le i \le n\\
    (1,x_{i-n}^{\theta,\bar{\xi}},f_{i}(\bar{u}))\, \text{ when } n+1 \le i \le 2n,
\end{cases}
\end{equation*}
where $f_1(\bar{u}),\dots f_n(\bar{u})$ are i.i.d. random variables with distribution $\lambda$ and $f_{n+1}(\bar{u}),\dots,f_{2n}(\bar{u})$ are i.i.d. random variables with distribution $\mathrm{Uniform}_\bA$. By construction, $\frac{1}{n}\sum_{i=n+1}^{2n}\mathcal{L}(z_i^{\theta,\bar{\xi}}) = \bar{\nu}^\theta$, and thus we know that
\begin{align*}
    \Delta_\mathcal{M}((2n)^{-\frac{1}{2}}) \le~& \sup_{\theta \in \Theta}\mathcal{R}(\Pi(P_\theta,\mu),\cH_k,(2n)^{-\frac{1}{2}},\nu^{\theta,\bar{\xi}})\\
    \le~& 2\sup_{\theta \in \Theta}\mathcal{R}(\Pi(P_\theta,\mu),\cH_k,n^{-\frac{1}{2}},\bar{\nu}^\theta)
\end{align*}
must decay fast with respect to $n$.

In the following, we establish the connection between the kernel's eigenvalues and perturbation response. The following proposition is the core of this connection.
{
\begin{prop}\label{lem_2}
Assume that 
\begin{equation*}
    \sup_{z \in \bS \times \bA} k(z,z) \le 1.
\end{equation*}
For any $\rho \in \mathcal{P}(\bS\times\bA)$, define
\begin{equation}\label{definition_n_rho}
    n(\rho) =  \max\{ i \in \mathbb{N}^+: n \Lambda_i^{\rho} \ge 1\}.
\end{equation}
We have
\begin{equation}\label{lem2_conclusion}
     \mathcal{R}(\mathcal{P}(\bS\times\bA),\cH_k,n^{-\frac{1}{2}},\nu) \ge \frac{1}{2}\sqrt{\sup_{\rho \in \mathcal{P}(\bS\times \bA)}\sum_{i = n(\nu)+1}^{+\infty}\Lambda_i^{\rho}},
\end{equation}
and, by $n(\nu) \le n$,
\begin{equation}
    \inf_{\nu \in \mathcal{P}(\bS\times\bA)}\mathcal{R}(\mathcal{P}(\bS\times\bA),\cH_k,n^{-\frac{1}{2}},\nu) \ge \frac{1}{2}\sqrt{\sup_{\rho \in \mathcal{P}(\bS\times \bA)}\sum_{i = n+1}^{+\infty}\Lambda_i^{\rho}}.
\end{equation}
Moreover, if there exists a distribution $\hat{\nu} \in \mathcal{P}(\bS\times\bA)$ such that 
\begin{equation}\label{uniform_bounded_eigenfunc}
    \sup_{ i \in \mathbb{N}^+}\|\psi_i^{\hat{\nu}}\|_\infty < +\infty,
\end{equation}
then
\begin{equation}\label{upper_bound_eigenvalue}
    \mathcal{R}(\mathcal{P}(\bS\times\bA),\cH_k,n^{-\frac{1}{2}},\hat{\nu})\le2 \sqrt{\frac{n(\hat{\nu})}{n}+ \sum_{i=n(\hat{\nu})+1}^{_\infty}\Lambda_i^{\hat{\nu}}}\sup_{ i \in \mathbb{N}^+}\|\psi_i^{\hat{\nu}}\|_\infty .
\end{equation}
\end{prop}
}
\begin{rem}
Assumption \eqref{uniform_bounded_eigenfunc} is widely used in the literature; see, e.g.,~\cite{steinwart2009optimal,mendelson2010regularization,bach2017equivalence,yang2020function}. However, note that even for $C^\infty$-kernel, $\eqref{uniform_bounded_eigenfunc}$ does not always hold; see, e.g.,~\cite{zhou2002covering}. 
To see a concrete example of lower bound \eqref{lem2_conclusion} and upper bound \eqref{upper_bound_eigenvalue}, assume $\Lambda_i^{\hat{\nu}} =\Theta( i^{-\alpha})$ for $\alpha >1$, and we have
\begin{equation*}
   \mathcal{R}(\mathcal{P}(\bS\times\bA),\cH_k,n^{-\frac{1}{2}},\hat{\nu})=  \Theta( n^{-\frac{\alpha-1}{2\alpha}}).
\end{equation*}

\end{rem}
\begin{proof}
We first prove that $n(\rho) \le n$ for any $\rho \in \mathcal{P}(\bS\times\bA)$.
When $i \ge n+1$, using the Mercer decomposition \eqref{mercer_decompo} and the fact that $\{\psi_i^\nu\}_{i \in \mathbb{N}^+}$ is orthonormal in $L^2(\nu)$, we have,
\begin{equation*}
    n\Lambda_i^{\nu} \le \sum_{j=1}^n\Lambda_j^{\nu}\le \sum_{j=1}^{+\infty}\Lambda_j^{\nu} = \int_{\bS\times\bA}k(z,z)\rmd \nu(z) \le 1.
\end{equation*}
By the definition \eqref{definition_n_rho}, we know that $n(\rho) \le n$.

We will use the expression~\eqref{kernel_dual}
\begin{align}
    \mathcal{R}( \mathcal{P}(\bS\times\bA),\cH_k,\epsilon,\nu) = \sup_{\rho \in  \mathcal{P}(\bS\times\bA)}\inf_{g \in L^2(\nu)}
   [\mathrm{MMD}_k(\rho, g\circ \nu) + \epsilon\|g\|_{L^2(\nu)}]
\end{align} 
to prove our result.
First, given $\nu \in \mathcal{P}(\bS\times\bA)$ and $\rho_0 \in \mathcal{P}(\bS\times\bA)$, we can use Mercer's decomposition \eqref{mercer_decompo} to obtain
\begin{align*}
    &\inf_{g \in L^2(\nu)}(\mathrm{MMD}(\rho_0, g\circ \nu) + \frac{\|g\|_{L^2(\nu)}}{\sqrt{n}})^2 \notag\\
    \ge& \inf_{g \in L^2(\nu)}[\mathrm{MMD}^2(\rho_0, g\circ \nu) + \frac{\|g\|_{L^2(\nu)}^2}{n}] \notag\\
    =& \inf_{g \in L^2(\nu)}[\int_{\bS \times \bA}\int_{\bS \times \bA}k(z,z')\rmd (\rho_0 - g\circ \nu)(z) \rmd (\rho_0 - g\circ \nu)(z') + \frac{\|g\|_{L^2(\nu)}^2}{n}] \notag\\
    =& \inf_{g \in L^2(\nu)}\sum_{i=1}^{+\infty}[\Lambda_i^{\nu}(\int_{\bS \times \bA}\psi_i^{\nu}(z)\rmd \rho_0(z) - \int_{\bS \times \bA}\psi_i^{\nu}(z) g(z)\rmd \nu(z))^2 + \frac{1}{n}(\int_{\bS \times \bA}\psi_i^{\nu}(z) g(z)\rmd \nu(z))^2]\notag\\
    =&\inf_{\{g_i\}_{i \in \mathbb{N}^+}}\sum_{i=1}^{+\infty}[\Lambda_i^{\nu}(\int_{\bS \times \bA}\psi_i^{\nu}(z)\rmd \rho_0(z) - g_i)^2 + \frac{1}{n}|g_i|^2] \notag\\
    =&\sum_{i=1}^{+\infty}\frac{\Lambda_i^\nu(\int_{\bS\times\bA}\psi_i^\nu(z)\rmd\rho_0(z))^2}{n\Lambda_i^\nu + 1}.
\end{align*}
Similarly, we have
\begin{align*}
    \inf_{g \in L^2(\hat{\nu})}(\mathrm{MMD}(\rho_0, g\circ \hat{\nu}) + \frac{\|g\|_{L^2(\hat{\nu})}}{\sqrt{n}})^2 &\le 2\sum_{i=1}^{+\infty}\frac{\Lambda_i^{\hat{\nu}}(\int_{\bS\times\bA}\psi_i^{\hat{\nu}}(z)\rmd\rho_0(z))^2}{n\Lambda_i^{\hat{\nu}} + 1} \\
    &\le 2\sum_{i=1}^{+\infty}\frac{\Lambda_i^{\hat{\nu}}}{n\Lambda_i^{\hat{\nu}}+1}\sup_{i \in \mathbb{N}^+}\|\psi_i^{\hat{\nu}}\|_\infty^2.
\end{align*}
Noticing that
\begin{equation*}
    \sum_{i=1}^{+\infty}\frac{\Lambda_i^{\hat{\nu}}}{n\Lambda_i^{\hat{\nu}}+1} \le 2[\frac{n(\hat{\nu})}{n} +\sum_{i = n(\hat{\nu})+1}^{+\infty}\Lambda_i^{\hat{\nu}}],
\end{equation*}
we obtain inequality \eqref{upper_bound_eigenvalue}.

We then prove inequality \eqref{lem2_conclusion}.
Given any $z_0 \in \bS \times \bA$, picking $\rho_0 = \delta_{z_0}$, we have
\begin{equation*}
    \inf_{g \in L^2(\nu)}[\mathrm{MMD}(\delta_{z_0}, g\circ \nu) + \frac{\|g\|_{L^2(\nu)}}{\sqrt{n}}]^2 \ge \sum_{i=1}^{+\infty}\frac{\Lambda_i^\nu(\psi_i^\nu(z_0))^2}{n\Lambda_i^\nu + 1}.
\end{equation*}
Therefore, for any $\rho \in \mathcal{P}(\bS \times \bA)$, by taking average with respect to $z_0$, we have
\begin{equation*}
    \sup_{\nu' \in \mathcal{P}(\bS \times \bA)}\inf_{g \in L^2(\nu)}[\mathrm{MMD}(\nu', g\circ \nu) + \frac{\|g\|_{L^2(\nu)}}{\sqrt{n}}]^2 \ge \sum_{i=1}^{+\infty}\frac{\Lambda_i^\nu \int|\psi_i^\nu(z)|^2\rmd \rho(z)}{n\Lambda_i^\nu + 1}.
\end{equation*}
Noticing that $n\Lambda_i^\nu \le 1$ when $i \ge n(\nu) + 1$ by definition \ref{definition_n_rho}, we have
\begin{equation}\label{lower_bound_with_l2_norm}
    \sup_{\nu' \in \mathcal{P}(\bS \times \bA)}\inf_{g \in L^2(\nu)}[\mathrm{MMD}(\nu', g\circ \nu) + \frac{\|g\|_{L^2(\nu)}}{\sqrt{n}}]^2 \ge \frac{1}{2} \sum_{i=n(\nu)+1}^{+\infty}\Lambda_i^\nu \int_{\bS \times \bA}|\psi_i^\nu(z)|^2\rmd \rho(z).
\end{equation}
Again, using Mercer decomposition \eqref{mercer_decompo} and the eigenvalues $\{\Lambda_i^\rho\}_{i \in \mathbb{N}^+}$ corresponding to the operator $\mathcal{K}_\rho$, we have 
\begin{equation}\label{mercer_sum_equation}
    \sum_{i=1}^{+\infty}\Lambda_i^\nu \int_{\bS\times\bA}|\psi_i^{\nu}(z)|^2\rmd \rho(z) = \int_{\bS\times\bA}k(z,z)\rmd \rho(z) = \sum_{i=1}^{+\infty}\Lambda_i^\rho.
\end{equation}
Checking~\eqref{lower_bound_with_l2_norm}\eqref{mercer_sum_equation} together, we need to have an upper bound of $\sum_{i=1}^{n(\nu)}\Lambda_i^\nu \int_{\bS \times \bA}|\psi_i^\nu(z)|^2\rmd \rho(z)$.
Let 
\begin{equation*}
    c_j = \sum_{i=1}^{n(\nu)}\Lambda_i^\nu(\int_{\bS\times\bA}\psi_i^\nu(z)\psi_j^\rho(z)\rmd\rho(z))^2.
\end{equation*}
By the Parserval's equality in $L^2(\rho)$, we have 
\begin{equation}\label{fourier_expansion}
    \sum_{i=1}^{n(\nu)}\Lambda_i^\nu \int_{\bS \times \bA}|\psi_i^\nu(z)|^2\rmd \rho(z) = \sum_{i=1}^{n(\nu)}\sum_{j=1}^{+\infty}\Lambda_i^\nu(\int_{\bS\times\bA}\psi_i^\nu(z)\psi_j^\rho(z)\rmd\rho(z))^2 = \sum_{j=1}^\infty c_j.
\end{equation}
Note that
\begin{align}
   c_j &= \sum_{i=1}^{n(\nu)}\Lambda_i^{\nu}(\int_{\bS\times\bA}\psi_i^\nu(z)\psi_j^\rho(z)\rmd\rho(z))^2 \\
   &\le \sum_{i=1}^{+\infty}\Lambda_i^{\nu}(\int_{\bS\times\bA}\psi_i^\nu(z)\psi_j^\rho(z)\rmd\rho(z))^2 \notag\\
    &=\int_{\bS\times\bA}\int_{\bS\times\bA}\psi_j^\rho(z)\psi_j^\rho(z')\sum_{i=1}^{+\infty}\Lambda_i^\nu\psi_i^\nu(z)\psi_i^\nu(z')\rmd \rho(z)\rmd\rho(z') \notag\\
    &=\int_{\bS\times\bA}\int_{\bS\times\bA}\psi_j^\rho(z)\psi_j^\rho(z')k(z,z')\rmd \rho(z)\rmd\rho(z') = \Lambda_j^\rho,
\end{align}
and, by equation~\eqref{mercer_norm},
\begin{equation}
    n(\nu) = \sum_{i=1}^{n(\nu)}\Lambda_i^\nu \|\psi_i^\nu\|^2_{\mathcal{H}_k} = \sum_{i=1}^{n(\nu)} \sum_{j=1}^{+\infty}\frac{\Lambda_i^\nu (\int_{\bS\times\bA} \psi_i^\nu(z)\psi_j^\rho(z)\rmd \rho(z))^2}{\Lambda_j^\rho} = \sum_{j=1}^{+\infty}\frac{c_j}{\Lambda_j^\rho}.
\end{equation}
Combining the last two inequality, we have
\begin{equation}
\label{eq:linear_constraints}
    0 \le c_j \le \Lambda_j^\rho, \, \quad \sum_{j=1}^{+\infty}\frac{c_j}{\Lambda_j^\rho}  = n(\nu).
\end{equation}
Under the constraint~\eqref{eq:linear_constraints} and the condition $\{\Lambda_j^\rho\}_{j \in \mathbb{N}^+}$ being nonincreasing, we know $\sum_{j=1}^{+\infty}c_j$ achieves the maximum value when $c_j=\Lambda_j^\rho$ for $j=1,\dots,n(\nu)$ and $c_j=0$ for $j\geq n(\nu)+1$. Therefore, by equation \eqref{fourier_expansion}, we know that
\begin{equation}
     \sum_{i=1}^{n(\nu)}\Lambda_i^\nu \int_{\bS \times \bA}|\psi_i^\nu(z)|^2\rmd \rho(z) = \sum_{j=1}^{+\infty}c_j \le \sum_{j=1}^{n(\nu)}\Lambda_j^\rho.
\end{equation}
Combining the last inequality with~\eqref{lower_bound_with_l2_norm}\eqref{mercer_sum_equation}, we have
\begin{equation}
    \sup_{\nu' \in \mathcal{P}(\bS \times \bA)}\inf_{g \in L^2(\nu)}[\mathrm{MMD}(\nu', g\circ \nu) + \frac{\|g\|_{L^2(\nu)}}{\sqrt{n}}]^2 \ge \frac{1}{2} \sum_{i=n(\nu)+1}^{+\infty}\Lambda_j^\rho .
\end{equation}
We finish our proof by noticing that both $\nu$ and $\rho$ are arbitrary probability distributions.
\end{proof}

The above proposition shows that the perturbation response $\mathcal{R}(\mathcal{P}(\bS\times\bA),\cH_k,n^{-\frac{1}{2}},\nu)$ is bounded by the eigenvalue decay.
On the one hand, if the eigenvalue decay of $\mathcal{K}_\nu$ is fast, even when $\Pi$ includes all probability distributions on $\bS\times\bA$, the decay of $\mathcal{R}(\Pi,\cH_k,n^{-\frac{1}{2}},\nu)$ with respect to $n$ is fast, which explains the positive results for RL algorithms in RKHS established in \cite{yang2020provably,yang2020function}. 
On the other hand, we note that the kernel function can be defined through Mercer decomposition, with the only requirement of $\{\Lambda_i^\rho\}_{i\in \mathbb{N}^+}$ being $\sum_{i=1}^{+\infty}\Lambda_i^\rho < +\infty$. Therefore, without further assumption on the kernel $k$, we only know that the right-hand side of \eqref{lem2_conclusion} converges to zero when $n$ goes to infinity, but its speed can be arbitrary slow.
In fact, for many popular RKHSs like the RKHSs corresponding to Laplace kernel and neural tangent kernel on sphere $\mathbb{S}^{d-1}$, the right-hand side of \eqref{lem2_conclusion} can be bounded below by $n^{-\frac{1}{d-1}}$; see \cite[Section~5]{long20212} for discussion.
For those RKHSs, the knowledge of $\Pi$, such as condition \eqref{prop3_condition}, plays a vital role in designing efficient RL algorithms. %
Without such knowledge of $\Pi$, there exist many RL problems whose sample complexity suffers from the curse of dimensionality.
Below we give two such examples.
\bigskip

\noindent\textbf{Single State, High Dimensional Action Space.} We first consider a problem in which the state space $\bS$ consists of a single point $s_0$ while the action space $\bA$ is $\mathbb{S}^{d-1}$ and $H = 1$. In this setting, the RL problem is essentially to find the maximum value of the reward function lying in the unit ball of $\mathcal{H}_k$ based on the values of $n$ points. Based on Proposition \ref{lem_2}, the convergence rate can be bounded below by the eigenvalue decay. Therefore, if we consider the RKHS corresponding to the Laplacian kernel or neural tangent kernel, the convergence rate suffers from the curse of dimensionality. We can then conclude that if we want to solve RL problems with high dimensional action space, we need to assume the decay of eigenvalue is fast enough to break the curse of dimensionality.

\bigskip
\noindent\textbf{High Dimensional State, Finite Action Space.} Even when the action space is finite, there still exist MDPs that cannot be solved without the curse of dimensionality. For any dimension $d \ge 2$, length of each episode $H \in \mathbb{N}^+$ and positive constant $\delta > 0$, we define an MDP family $\mathcal{M}_{d,H,\delta}$ as follows:
\begin{align*}
    &\mathcal{S} = \mathbb{S}^{d-1}, \quad \mathcal{A} = \{0,1\}, \quad H = H, \quad \Theta_P = \{0\}, \quad \mu = \text{Uniform}_{\mathbb{S}^{d-1}},\\
    &\{r_{\theta_r}:\theta_r \in \Theta_r\} = \{r: \|r(h,\cdot)\|_{\mathcal{H}_k} \le 1, \forall h \in [H]\}, \\ &k((s,a),(s',a')) = \exp(-\|s-s'\|),  \quad 
    P_0(\,\cdot\,|\,h,s,a) = \delta_{T_{a,h} s}(\cdot), \\
    &T_{a,h} s = \begin{cases}
       (\phi_1,\dots,\phi_{h_d} + \delta,\dots,\phi_d), \text{ when } a = 0,\\
        (\phi_1,\dots,\phi_{h_d} - \delta,\dots,\phi_d), \text{ when } a = 1,
    \end{cases} 
\end{align*}
where $h_d = h\, mod\, d$ and we use the spherical coordinates $(\phi_1,\dots,\phi_d)$ to denote the points on $\mathbb{S}^{d-1}$.
Notice that for every $d, H, \delta$, the transition probability in the RL problem $\mathcal{M}_{d,H,\delta}$ is known since $\Theta_P$ is a single-point set. By construction, the transition probability $P_0$ indicates that the agent can take an action at each step to move on the sphere surface along one spherical coordinate (depending on $h$) with size $\delta$. When $\delta$ is small and $H$ is large, we know that for any delta function on $\mathbb{S}^{d-1}$, we can find a policy such that the corresponding state distribution at step $H$ is close to that delta function. In other words, the set $\Pi(H,P_0,\mu)$ is very large. Noticing that the kernel is a Laplacian kernel whose eigenvalue decay is slow, we expect the RL problem to be difficult to solve.
The following theorem shows that there exists no dimension-free bound on $\Delta_{\mathcal{M}_{d,H,\delta}}(n^{-\frac{1}{2}})$ for all $n, d, H$, and $\delta$.
Therefore, the above RL problem can not be solved without the curse of dimensionality.
\begin{prop}\label{cod_case}
There exist no universal constants $\alpha, \beta > 0$ and constant $C_d > 0$ only depending on $d$ such that
\begin{equation*}
   \sup_{\delta > 0}\Delta_{\mathcal{M}_{d,H,\delta}}(n^{-\frac{1}{2}}) \le C_d H^\alpha (\frac{1}{n})^\beta
\end{equation*}
holds for all $n, H \in \mathbb{N}^{+}$ and $d \ge 2$.
\end{prop}
\begin{proof}
For any $s \in \mathbb{S}^{d-1}$, we define the following policy $\pi^s$
\begin{equation*}
    \pi_h^s(\,\cdot\,|\,s'): \text{ choose  the policy } a \text{ such that } T_{a,h}s' \text{ is  closest to } s \text{ among all possible choices.} 
\end{equation*}
Then, in the above MDP with policy $\pi^s$, when $h \ge \frac{C'_d}{\delta}$, we have
\begin{equation}\label{concentration on a point}
    \|S_h -s \| \le C'_d \delta \text{ with probability 1}.
\end{equation}
Here $C'_d > 0$ is a constant only depending on $d$, which may vary from line to line in the following proof.
Using Proposition~\ref{lem_2} and the eigenvalue decay of the Laplacian kernel (see, e.g., \cite[Section 5]{long20212}), we know that for any $\nu \in \mathcal{P}(\bS\times\bA)$,
\begin{align*}
      &\sup_{z \in \bS \times \bA}\inf_{g \in L^2(\nu)}[\mathrm{MMD}_k(\delta_z, g\circ \nu) + \frac{\|g\|_{L^2(\nu)}}{\sqrt{n}}] \\
     = &\sup_{\rho \in \mathcal{P}(\bS \times \bA)}\inf_{g \in L^2(\nu)}[\mathrm{MMD}_k(\rho, g\circ \nu) + \frac{\|g\|_{L^2(\nu)}}{\sqrt{n}}] \\
     \ge & C'_d n^{-\frac{1}{2(d-1)}}.
\end{align*}
Therefore, there exists $(s^*,a^*) \in \bS \times\bA$ such that
\begin{equation*}
    \inf_{g \in L^2(\nu)}[\mathrm{MMD}_k(\delta_{(s^*,a^*)}, g\circ \nu) + \frac{\|g\|_{L^2(\nu)}}{\sqrt{n}}] \ge C'_d n^{-\frac{1}{2(d-1)}}.
\end{equation*}
Combining the last equation with Lemma \ref{thm: concentration_coefficient}, if $H = [\frac{C'_d}{\delta}]$, we have
\begin{align*}
    \mathcal{R}(\Pi(H,P_0,\mu),\cH_k,n^{-\frac{1}{2}},\nu) &\ge  \inf_{g \in L^2(\nu)}[\mathrm{MMD}_k(\rho_{H,P_0,\pi^{s^*},\mu}, g\circ \nu) + \frac{\|g\|_{L^2(\nu)}}{\sqrt{n}}]\\
    &\ge C'_d n^{-\frac{1}{2(d-1)}} - \mathrm{MMD}_k(\delta_{(s^*,a^*)},\rho_{H,P_0,\pi^{s^*},\mu}).
\end{align*}
By inequality \eqref{concentration on a point}, we have,
\begin{equation*}
    \mathrm{MMD}_k(\delta_{(s^*,a^*)},\rho_{H,P_0,\pi^{s^*},\mu}) \le C'_d\delta.
\end{equation*}
Therefore, we can choose
\begin{equation*}
    \delta = C'_d n^{-\frac{1}{2(d-1)}},
\end{equation*} 
such that
\begin{equation*}
   \mathcal{R}(\Pi(H,P_0,\mu),\cH_k,n^{-\frac{1}{2}},\nu)\ge C'_d n^{-\frac{1}{2(d-1)}}
\end{equation*}
and
\begin{equation*}
   H = [C'_dn ^{\frac{1}{2(d-1)}}]. 
\end{equation*}
Therefore, combining the last two equations and the definition of $\Delta_{\mathcal{M}_{d,H,\delta}}(n^{-\frac{1}{2}})$ in the case of known transition \eqref{eq:definition_delta1}, if the constants $\alpha,\beta$ and $C_d$ exist, we must have
\begin{equation}
    C'_d n^{-\frac{1}{2(d-1)}} \le C_d (C'_d)^\alpha n^{\frac{\alpha}{2(d-1)}-\beta} 
\end{equation}
holds for all $n\in \mathbb{N}^+$ and $d\ge2$. Therefore,
\begin{equation}
    \frac{\alpha+1}{2(d-1)} \ge \beta
\end{equation}
holds for all $d\ge 2$, 
which is a contradiction.
\end{proof}
\section{Conclusions and Future Works}
In this paper, we define the perturbational complexity by distribution mismatch $\Delta_{\mathcal{M}}(\epsilon)$ when the reward functions lie in the unit ball of an RKHS and the transition probabilities lie in a given arbitrary set. We show that $\Delta_{\mathcal{M}}(\epsilon)$ is an informative indicator of whether the RL problems can be solved efficiently or not. Some concrete properties of $\Delta_{\mathcal{M}}(\epsilon)$ are studied in several cases. There are still quite a few unsolved problems related to this topic. First, the upper bound and lower bound in the current work are not matched. Second, in the case of unknown transition, we still need assumption \eqref{Bellman_assumption} to prove the convergence of fitted Q-iteration algorithms. How to relax this assumption or show the necessity of this assumption remains unclear yet. Third, our lower bound mainly utilizes the uncertainty of reward functions. If the reward function is known or the reward function is deterministic such that there is no noise in reward, our lower bound can not be applied. It is of interest to study the lower bound and corresponding upper bound %
in these situations. Fourth, our setup of RL requires a generative simulator. In some RL problems, however, one can only access an episodic simulator. How to establish similar results in this case is still open. Finally, we wish to use information related to the perturbational complexity to guide the design of efficient RL algorithms in practice.

\section*{Acknowledgement}
We thank Professor Weinan E and Professor Mengdi Wang for their valuable comments and suggestions during the preparation of this work.

\bibliographystyle{plain}
\bibliography{reference}
\end{document}